%% file: paper_template.tex
\documentclass[conference]{IEEEtran}
\usepackage{times}

\usepackage[numbers,sort&compress]{natbib}
\usepackage[svgnames]{xcolor}
\usepackage{hyperref}
\hypersetup{colorlinks,
         breaklinks,
         urlcolor=RoyalBlue,
         linkcolor=RoyalBlue,
         citecolor=Teal,
         }

\usepackage{graphicx, float, url, wrapfig}
\usepackage{amsmath, amssymb, bbm, mathtools}
\usepackage{siunitx}
\usepackage[capitalise, nameinlink]{cleveref}
\usepackage{bibunits}
\restylefloat{table}
\usepackage{dsfont} 

\usepackage{algorithm}
\usepackage{algpseudocode}
\usepackage{booktabs}
\usepackage{multirow}
\usepackage{multicol}
\let\labelindent\relax  
\usepackage{enumitem}
\providecommand{\ie}{\emph{i.e.,} }
\providecommand{\eg}{\emph{e.g.,} }

\usepackage{csquotes}

\usepackage[font=small]{caption}

\usepackage{makecell, tabularx, multirow,  adjustbox}
\usepackage{spverbatim}

\newcommand{\ols}[1]{\mskip.5\thinmuskip\overline{\mskip-.5\thinmuskip {#1} \mskip-.5\thinmuskip}\mskip.5\thinmuskip} 
\newcommand{\olsi}[1]{\,\overline{\!{#1}}} 
\makeatletter
\newcommand\closure[1]{
  \tctestifnum{\count@stringtoks{#1}>1} 
  {\ols{#1}} 
  {\olsi{#1}} 
}
\long\def\count@stringtoks#1{\tc@earg\count@toks{\string#1}}
\long\def\count@toks#1{\the\numexpr-1\count@@toks#1.\tc@endcnt}
\long\def\count@@toks#1#2\tc@endcnt{+1\tc@ifempty{#2}{\relax}{\count@@toks#2\tc@endcnt}}
\def\tc@ifempty#1{\tc@testxifx{\expandafter\relax\detokenize{#1}\relax}}
\long\def\tc@earg#1#2{\expandafter#1\expandafter{#2}}
\long\def\tctestifnum#1{\tctestifcon{\ifnum#1\relax}}
\long\def\tctestifcon#1{#1\expandafter\tc@exfirst\else\expandafter\tc@exsecond\fi}
\long\def\tc@testxifx{\tc@earg\tctestifx}
\long\def\tctestifx#1{\tctestifcon{\ifx#1}}
\long\def\tc@exfirst#1#2{#1}
\long\def\tc@exsecond#1#2{#2}
\makeatother

\newtheorem{claim}{\bf Claim}
\newtheorem{proposition}{\bf Proposition}

\newif\ifarxiv
\arxivtrue

\newif\ifreview
\reviewfalse

\begin{document}

\title{Explore until Confident: Efficient Exploration for Embodied Question Answering}

\ifarxiv
\author{
\authorblockN{Allen Z. Ren\authorrefmark{1}
 \ \ \ \ Jaden Clark\authorrefmark{2} \ \ \ \ Anushri Dixit\authorrefmark{1}\ \ \ \ Masha Itkina \authorrefmark{3}\ \ \ \ Anirudha Majumdar\authorrefmark{1}\ \ \ \ Dorsa Sadigh\authorrefmark{2}}
\authorblockA{
\authorrefmark{1}
Princeton University
\authorrefmark{2}
Stanford University
\authorrefmark{3}
Toyota Research Institute\\
\url{https://explore-eqa.github.io}
}}
\else
\author{Author Names Omitted for Anonymous Review. Paper-ID 229.}
\fi
\maketitle

\begin{abstract}

We consider the problem of \emph{Embodied Question Answering (EQA)}, which refers to settings where an embodied agent such as a robot needs to actively explore an environment to gather information until it is confident about the answer to a question. 
In this work, we leverage the strong semantic reasoning capabilities of large vision-language models (VLMs) to efficiently explore and answer such questions.
However, there are two main challenges when using VLMs in EQA: they do not have an internal memory for mapping the scene to be able to plan how to explore over time, and their confidence can be miscalibrated and can cause the robot to prematurely stop exploration or over-explore. We propose a method that first builds a semantic map of the scene based on depth information and via visual prompting of a VLM --- leveraging its vast knowledge of relevant regions of the scene for exploration. 
Next, we use conformal prediction to calibrate the VLM's question answering confidence, allowing the robot to know when to stop exploration --- leading to a more calibrated and efficient exploration strategy. To test our framework in simulation, we also contribute a new EQA dataset with diverse, realistic human-robot scenarios and scenes built upon the Habitat-Matterport 3D Research Dataset (HM3D). Both simulated and real robot experiments show our proposed approach improves the performance and efficiency over baselines that do no leverage VLM for exploration or do not calibrate its confidence.
\end{abstract}

\IEEEpeerreviewmaketitle


\input{sections/intro}
\input{sections/related}
\input{sections/formulation}
\input{sections/approach-map}
\input{sections/approach-cp}

\input{sections/dataset}
\input{sections/experiment}
\input{sections/conclusion}

\ifarxiv
\section*{Acknowledgments}
We thank Donovon Jackson, Derick Seale, and Tony Nguyen for contributing to the HM-EQA dataset. The authors were partially supported by the Toyota Research Institute (TRI), the NSF CAREER Award [\#2044149], the Office of Naval Research [N00014-23-1-2148], the NSF Award [\#1941722], the ONR Award [\#N00014-22-1-2293], the DARPA grant [\#W911NF2210214], and Princeton SEAS Innovation Award from The Addy Fund for Excellence in Engineering. This article solely reflects the opinions and conclusions of its authors and NSF, ONR, DARPA, TRI or any other Toyota entity. This research was supported in part by Other Transaction award HR00112490375 from the U.S. Defense Advanced Research Projects Agency (DARPA) Friction for Accountability in Conversational Transactions (FACT) program.
\fi


\bibliographystyle{unsrtnat}
\bibliography{references}


\renewcommand{\thetable}{A\arabic{table}}
\renewcommand{\theequation}{A\arabic{equation}}
\renewcommand\thefigure{A\arabic{figure}}
\renewcommand{\thesubsection}{A\arabic{subsection}}
\setcounter{figure}{0}
\setcounter{table}{0}
\setcounter{equation}{0}


\clearpage
\begin{appendices}
\input{sections/appendix.tex}
\end{appendices}


\end{document}

%% file: sections/intro.tex
\section{Introduction}
\label{sec:introduction}

Imagine that a service robot is sent to a home to perform various tasks, and the household owner asks it to check whether the stove is turned off.
This setting is referred to as Embodied Question Answering (EQA)~\cite{das2018embodied, gordon2018iqa}, where the robot starts at a random location in a 3D scene, explores the space, and stops when it is confident about answering the question. This can be a challenging problem due to highly diverse scenes
and lack of an a-priori map of the environment. Previous works rely on training dedicated exploration policies and question answering modules from scratch, which often only consider simple and syntactically structured questions such as ``what is the color of the sofa'' \cite{yu2019multi, das2018neural, wijmans2019embodied}. The models studied in prior work largely consider synthetic scenes, and can be data-inefficient since the training is done from scratch.


\begin{figure*}[t]
\begin{center}
\includegraphics[width=1.0\textwidth]{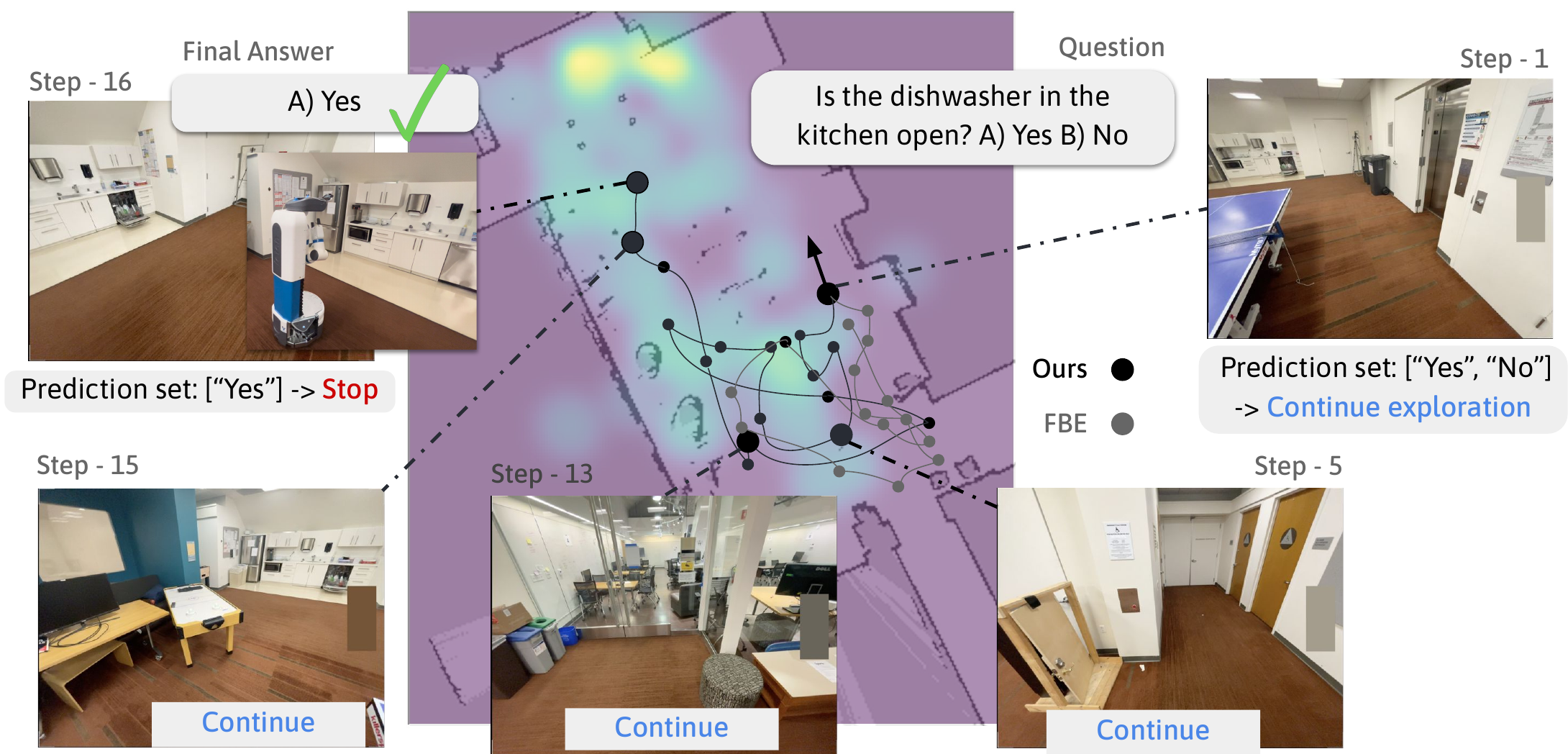}
\caption{Given a question about the scene (``Is the dishwasher in the kitchen open? A) Yes B) No''), our framework leverages a large vision-language model (VLM) to obtain semantic information from the views (visualized by overlaying it on top of the occupancy map), which guides a Fetch robot to explore relevant locations. Using such a semantic map helps robot explore more efficiently compared to Frontier-based exploration without using any semantic value (FBE, \cref{sec:experiments}). The robot maintains a set of possible answers and stops when the set reduces to a single answer based on the current view. In this example, the robot is confident at Step 16 where it sees the open dishwasher not too far from its position. The robot paths (thin lines) are approximated.}
\label{fig:anchor}
\vspace{-15pt}
\end{center}
\end{figure*}

Recently, large vision-language models (VLMs) have achieved impressive performance in answering complex questions about static 2D images that sometimes requires reasoning~\cite{li2023blip, liu2023visual}. They can also help the robot \emph{actively} perceive the 3D scene given partial 2D views and \emph{reason} about future actions for the robot to take~\cite{kwon2023toward}. Such capabilities are critical to performing EQA, as the robot can now better reason about relevant regions of the environment, actively explore them, and answer questions that require semantic reasoning (\eg answering ``what time is it now?'' by searching for a clock). However, 
there are two main challenges that arise in using VLMs for EQA in complex, diverse 3D scenes while trying to explore efficiently:
\begin{enumerate}[leftmargin=*]
    \item \textbf{Limited Internal Memory of VLMs.} Efficient exploration benefits from the robot tracking previously explored regions and also ones yet to be explored but \emph{relevant} for answering the question. However, VLMs do not have an internal memory for mapping the scene and storing such semantic information;
    \item \textbf{Miscalibrated VLMs}. VLMs are fine-tuned on pre-trained large language models (LLMs) as the language decoder, and LLMs have been shown to often be miscalibrated~\cite{kadavath2022language} -- that is they can be over-confident or under-confident about the output. This makes it difficult to determine when the robot is confident enough about question answering in EQA and then stop exploration, affecting overall efficiency.
\end{enumerate}

\emph{How can we endow VLMs -- with limited memory and potential for miscalibration -- with the capability of efficient exploration for EQA?} 
To address the first challenge, we construct a semantic map external to the VLM, combining the VLM's visual reasoning within the local view with the global geometric information of the map, and thus informing planning for the next waypoint.
To address the second challenge, we apply rigorous uncertainty quantification on the VLM's EQA predictions, such that the robot knows when it should stop to satisfy a certain level of prediction success. 

We propose a framework (\cref{fig:anchor}, \cref{fig:overview}) that leverages a VLM for answering open-ended questions in diverse 3D scenes by (1) fusing the commonsense/semantic reasoning abilities of a VLM into a global geometric map to enable efficient exploration, and (2)~using the theory of multi-step conformal prediction~\cite{vovk2005algorithmic,ren2023knowno} to formally quantify VLM uncertainty about the question. 
Through exploration, the robot builds a semantic map of the scene that stores information on occupancy and locations the VLM deems worth exploring. Such semantic information is obtained by annotating the free space in the current image view, prompting the VLM to choose among the unoccupied regions, and querying its prediction (\cref{fig:visual-prompting}). Heuristic planning is then applied to prioritize the robot exploring semantically relevant regions. Throughout an episode, the robot maintains a set of possible answers, updates the set at each step based on new visual information provided to the VLM, and stops exploration when the set of possible answers reduces to a single option. Conformal prediction formally ensures the set covers the true answer with high probability, and hence the robot can terminate exploration with calibrated confidence. Conformal prediction also minimizes the set size, and thus the robot can stop as soon as possible to avoid over-exploration.

For simulated experiments, we contribute HM-EQA, a new EQA dataset based on realistic human-robot scenarios and the Habitat-Matterport 3D Research Dataset (HM3D), which provides photo-realistic, diverse indoor 3D scans~\cite{ramakrishnan2021habitat}. We also perform hardware experiments in home/office-like environments using the Fetch mobile robot \cite{wise2016fetch}. Both simulated and hardware experiments show that our framework improves the EQA efficiency over baselines that do not use semantic information from VLM reasoning and do not calibrate the VLM for stopping criteria. 

%% file: sections/related.tex
\section{Related Work}
\label{sec:related-work}

Below we first introduce relevant work in the EQA problem setting, then other work that leverage VLM reasoning and semantic mapping, and lastly work that address the miscalibration issue especially using conformal prediction.

\smallskip \noindent \textbf{Embodied Question Answering.}
The EQA task (also known as IQA, Interactive Question Answering) was first introduced in 
\cite{das2018embodied, gordon2018iqa, yu2019multi, das2018neural, wijmans2019embodied} circa 2019, where the exploration policy and the question answering module are trained from scratch using data collected in largely synthetic indoor scenes. Due to the limited generalization capabilities of state-of-the-art models of the time, only simple questions regarding basic attributes of large objects (\eg color) were considered. 
Recently, a new line of work~\cite{ma2022sqa3d, hong20233dmv, hong20233dllm, huang2023embodied} leverages foundation models such as pre-trained LLMs or VLMs trained with 2D features sampled from 3D scenes to answer similar types of questions about the scene.
However, these recent works assume full information available (\eg full point cloud of the scene, or pre-sampled 2D images) for answering the question and do not consider the exploration problem.

\smallskip \noindent \textbf{VLMs for other embodied tasks.} VLMs have shown promise in solving other embodied tasks using their commonsense reasoning and knowledge~\cite{wen2023road, driess2023palm, sumers2023distilling, shen2023distilled}. VLMs have been fine-tuned to exhibit strong physical and spatial reasoning~\cite{chen2024spatialvlm, gao2023physically, brohan2023rt}. The closest work to ours include \cite{kwon2023toward}, where queries to a VLM and subsequently an LLM are used to help the robot reason about the most informative angle from which to view a scene, \eg the angle that helps determine whether a coffee cup is empty and needs to be cleaned up. Similarly \cite{shah2023navigation} uses an LLM to generate search heuristics guiding robot finding objects in unknown environments. Recently visual prompting with VLMs has been proposed to acquire VLM's prior knowledge about specific regions on 2D images \cite{nasiriany2024pivot, liu2024moka}, \eg searching for directions for reaching an object. In our work, we go beyond simple active exploration strategies generated by LLMs/VLMs, and carefully consider uncertainty estimation and active exploration while using the commonsense reasoning capabilities of VLMs to determine where to navigate in larger scenes. Previous work \cite{tian2022vibus} has also explored quantifying VLM uncertainty in semantic understanding, as well as leveraging VLM uncertainty for better anomaly detection in autonomous driving scenarios \cite{tian2023unsupervised}.

\smallskip \noindent \textbf{Semantic mapping with foundation models.} There is a large body of work in using foundation models to build a map containing information relevant to the task of interest (\ie semantic map \cite{chen2023not, dai2023think, kostavelis2016robot, jatavallabhula2023conceptfusion, gu2023conceptgraphs, chaplot2021seal}). NLMap~\cite{chen2023open}, VLMap~\cite{huang2023visual}, and CoW~\cite{gadre2022cow} encode image-text embeddings from models like CLIP~\cite{radford2021learning} or ViLD~\cite{gu2021open} into a spatial map, which is then used to find objects in the scene given a text query. LLMs are also employed to extract objects of interest from an open-ended query. CLIP-Fields~\cite{shafiullah2022clip} proposes a novel objective for mapping spatial locations to semantic embedding vectors using labels from these pre-trained models. However, these methods, when applied in navigation tasks, assume the semantic map is built offline, and the target location can be queried immediately given the text input. In our setting, the robot has no prior knowledge of the scene, and builds the semantic map online for guiding exploration. A closer work to ours is
\cite{zhou2023esc}, which uses a object grounding model to feed relevant information to LLM; LLM provide semantic information, which is then stored in a semantic map for guiding the exploration. Yokoyama et al.~\cite{yokoyama2023vlfm} uses the text-image relevance score from a VLM as the semantic information.
In contrast, our method directly leverages the probabilistic output of the VLM for semantic reasoning in a zero-shot manner without additional training or using another foundation model.

\smallskip \noindent \textbf{Calibration and conformal prediction.} With the rising popularity of LLMs (also VLMs fine-tuned from LLMs), many recent studies have investigated the possibility of them being miscalibrated~\cite{kadavath2022language, mielke2022reducing},
being over-confident or under-confident about their text outputs. 
This creates a unique challenge as these foundation models are increasingly applied in embodied tasks, and the agents can have miscalibrated confidence in their decisions, \eg when to stop exploration in EQA. To combat this, KnowNo~\cite{ren2023knowno} leverages conformal prediction (CP) to formally quantify the LLM's uncertainty in a robot planning setting 
, and ensures the robot plan is successful with calibrated confidence. In this work, we apply multi-step CP in EQA for determining when the VLM is confident about its answer in order to stop exploration in a multi-time-step setting (in contrast to KnowNo, which uses CP to decide when the robot should ask for help from a human). CP has also been utilized in deciding when to stop due to fault detection or unsafe behavior prediction \cite{lindemann2023conformal, pmlr-v211-dixit23a}. 

%% file: sections/formulation.tex
\section{Problem Formulation}
\label{sec:formulation}

\smallskip \noindent \textbf{Distribution of scenarios for EQA.} We formalize Embodied Question Answering (EQA) by considering an unknown joint distribution over \emph{scenarios} $\xi \sim \mathcal{D}$ the robot can encounter. A scenario is a tuple $\xi \coloneqq (e,T,g^0,q,y)$, where $e$ is a simulated or real 3D scene (\eg, a floor plan with certain dimensions), $T$ is the maximum number of time steps allowed for the robot to navigate in the scene (\eg a function of scene size), $g^0$ is the robot's initial pose (2D position and orientation at time 0),  $q$ is the question, and $y$ is the ground truth answer. We will use a subscript to indicate the scenario (\eg $T_\xi$ for the maximum time horizon in scenario $\xi$), and a superscript $t$ for time steps (\eg $g^t$ for the robot's pose at time $t$). 
In this work, we consider multiple-choice questions $q$, \eg ``Where did I leave the black suitcase? A) Bedroom B) Living room C) Storage room D) Dining room.'' We assume four choices for each question, and thus the set of labels $\mathcal{Y}:= \{\textrm`A\textrm', \textrm`B\textrm', \textrm`C\textrm', \textrm`D\textrm'\}$ contains any answer $y$.
We assume no knowledge of $\mathcal{D}$, except that we can sample a finite-size dataset of independent and identically distributed scenarios from it. 

\smallskip \noindent \textbf{Robot navigating in a scenario.} In this work, we would like a robot to be able to perform EQA in any given scenario~$\xi \in \mathcal{D}$. We do not expect the robot to have any prior knowledge of the scene. 
We initialize the robot at $g^\text{0}$, and at any time $t$ it can traverse to different poses $g^t$. The robot's onboard camera provides RGB images $I^t_c \in \mathbb{R}^{H_I \times W_I \times 3}$ and depth images $I^t_d \in \mathbb{R}^{H_I \times W_I}$. We associate a time step with each time the robot stops and takes RGB/depth images. Later we discuss how to select when and where the robot should take images --- for querying a VLM --- via an active exploration strategy. Additionally, we assume access to a collision-free planner $\pi$ that determines the next pose $g^{t+1}$ to travel to, a maximum of \SI{3}{\metre} away from $g^t$. We assume perfect odometry in simulation. In real-world settings, the robot can determine its new pose using a localization algorithm. 

\smallskip \noindent \textbf{VLM predictions.} A VLM pre-trained with large scale data provides information needed for solving the EQA task. We pass the RGB image and a text prompt $s$ to the model, and query its probability over predicting the next token. For convenience, we denote $x^t = (I^t_c, q)$ consisting of the RGB image $I^t_c$ and the question $q$. Then, the VLM's prediction given the question $q$ at time $t$ can be denoted as $\hat{f}(x^t) \in [0,1]^{|\mathcal{Y}|}$, which are the softmax scores over the multiple choice set $\mathcal{Y}$. We denote $\hat{f}_y(\cdot)$ as the softmax score for a particular label $y$. 


\smallskip \noindent \textbf{Goal: efficient exploration.} In a new scenario, the robot may stop at any time step $t \leq T_\xi$, and make a final answer to the question based on all the information (\eg VLM predictions over time steps). Our goal is to answer the question correctly in \emph{unseen} test scenarios $\xi \in \mathcal{D}$, using a minimal number of time steps. This requires the robot to search for relevant information efficiently without over-exploration.

\begin{figure}
\centering\includegraphics[width=0.98\linewidth]{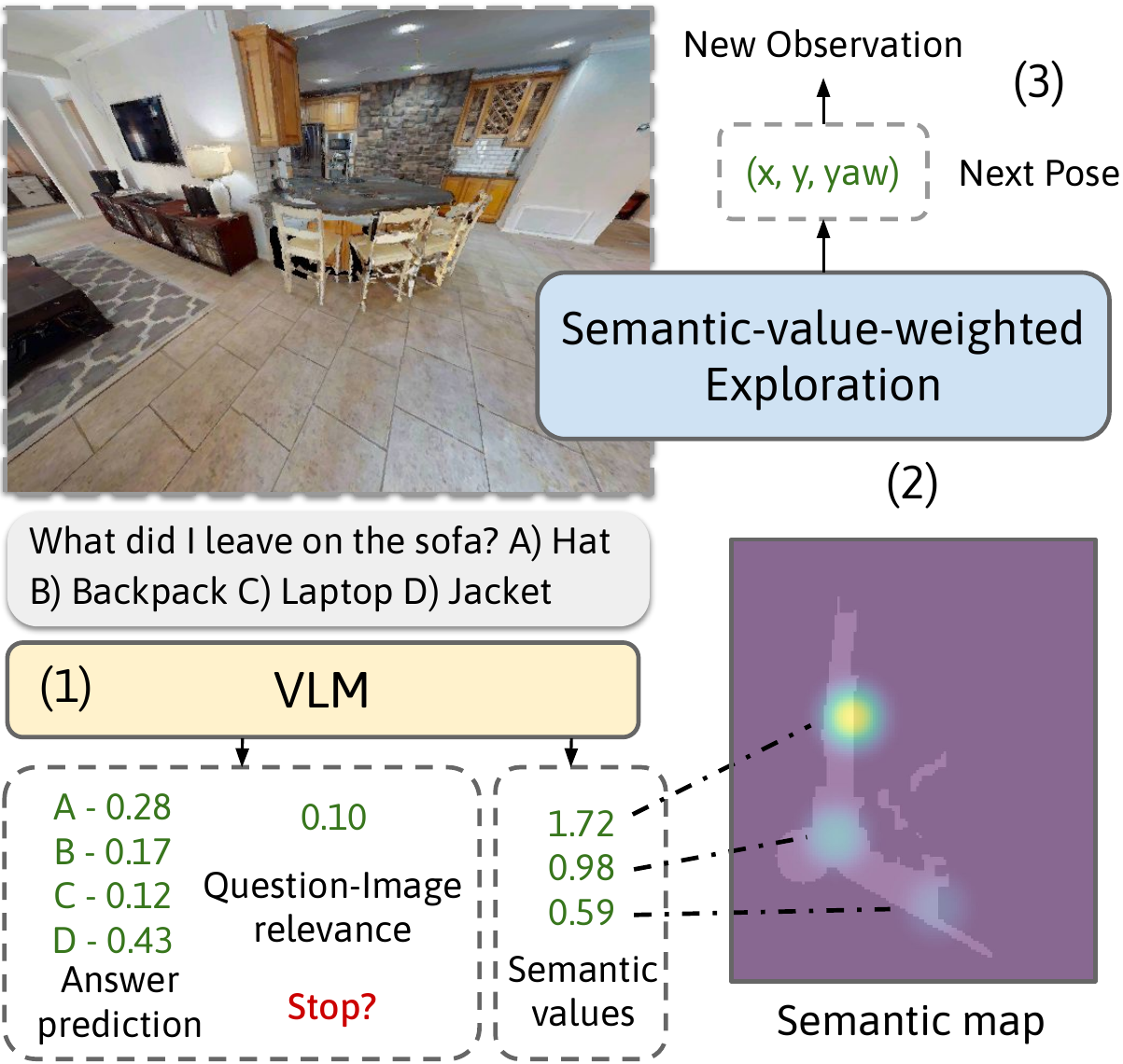}
    \caption{Overview of our framework for EQA tasks, which combines a VLM and a external semantic map for planning.}
    \label{fig:overview}
    \vspace{-10pt}
\end{figure}

%% file: sections/approach-map.tex
\section{Targeted Exploration using VLM Reasoning}

To improve exploration efficiency, we would like the robot to prioritize exploring regions relevant to answering the posed question. To this end, we propose using the rich knowledge from VLMs to guide exploration. However, as discussed earlier, VLMs have limited internal memory --- they are unable to keep track of past and future relevant scenes. We instead propose a novel solution for building a map of the scene external to the VLM, and embedding the VLM's knowledge about possible exploration directions into this map to guide the robot's exploration.

\subsection{Overview}
\cref{fig:overview} provides an overview of the full framework. Given the observation and the question, we first prompt a VLM to generate three different outputs: answer prediction probabilities over the four possible answers, the question-image relevance score relating how relevant the current view is for answering the question (\cref{sec:stopping-criterion}), and a set of semantic values (\cref{subsec:semantic-value}) indicating if any regions in the view are worth exploring for answering the question. These values are then stored in a semantic map external to the VLM, which also tracks free space and unknown regions. Then we apply semantic-value-weighted exploration based on the map (\cref{subsec:semantic-fbe}) that guides the robot exploring meaningful regions. The robot stops until it is confident about answering the question based on the answer prediction and question-image relevance.


\subsection{Exploration Map and Frontier-Based Exploration}
\label{subsec:map}
For tracking where the robot has explored, we first adopt a 3D voxel-based representation for the map of size $L \times W \times H$ --- $M$ and $L$ expand as the robot explores more areas, and $H$ is fixed as $\SI{3.5}{\metre}$ (typical floor height). Each voxel corresponds to a cube with side length $l$. 
At each pose $g^t$ with depth image $I^t_d \in \mathbb{R}^{H_I \times W_I}$ and known camera intrinsics, we apply Volumetric Truncated Signed Distance Function (TSDF) Fusion~\cite{newcombe2011kinectfusion, zeng20163dmatch} to update (1) occupancy of the voxels and (2) if they are explored/seen in the current $I^t_d$. While all voxels seen in $I^t_d$ are used to update occupancy, only those within a smaller field of view are used to update whether they have been explored, enabling  more fine-grained exploration (App.~\cref{app:experiment-details}). At each time step we then project the 3D voxel map into a 2D point map $M$: a 2D point is considered free (un-occupied) if all voxels up until $\SI{1.5}{\metre}$ are marked free, which is the height of the camera (in simulation and in reality), and considered explored if all voxels along $H$ have been marked explored.


Based on the 2D map storing occupancy and exploration information, we use a heuristics-based 2D planner that plans new poses around unexplored region. Our strategy is based around Frontier-Based Exploration (FBE), which has been proven a simple yet effective method for navigation tasks~\cite{yamauchi1997frontier}. FBE finds the \emph{frontiers} (\cref{fig:visual-prompting}), the locations at the boundary of the explored and unexplored regions, samples one as the planned location, and uses the normal direction to the unexplored region boundary as the planned orientation. 



\begin{figure*}[t]
\begin{center}
\includegraphics[width=0.95\textwidth]{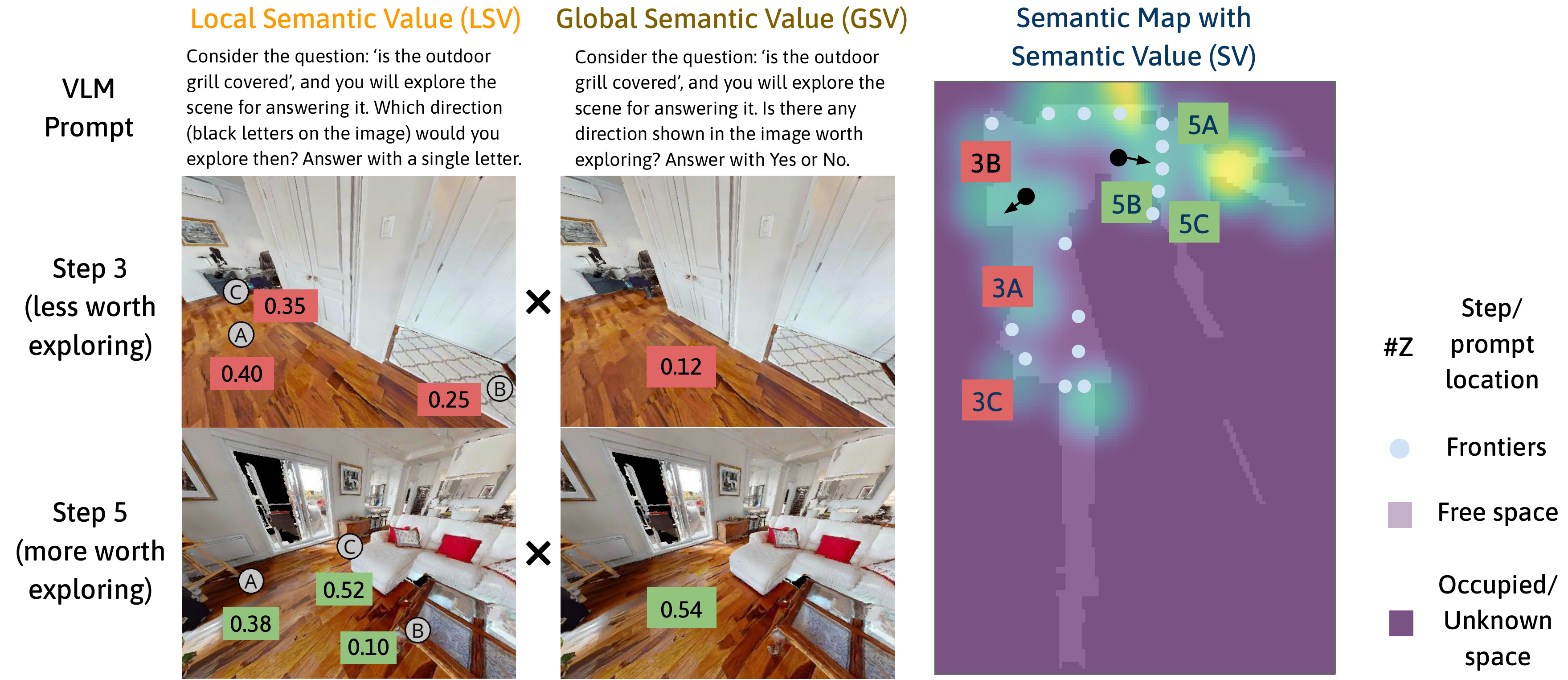}
\caption{To query VLM's uncertainty over possible exploration locations, we visually prompt the VLM with possible points in the current view (left column) and also with the entire view (middle column) to obtain the Local Semantic Value (LSV) and Global Semantic Value (GSV) (\cref{subsec:semantic-value}). A weighted combination of them (SV) is saved in a semantic map (right column). The values are used as the weights for sampling the next frontier to navigate to, guiding the robot towards \emph{unknown and relevant} regions (\cref{subsec:semantic-fbe})}
\label{fig:visual-prompting}
\vspace{-15pt}
\end{center}
\end{figure*}

\subsection{VLM Visual Prompting for Semantic Value (\cref{fig:visual-prompting})} 
\label{subsec:semantic-value}

As the VLM already has access to rich prior knowledge from large-scale Internet data, we hypothesize that it can potentially provide useful information in determining relevant locations to explore. We achieve this by obtaining the VLM's uncertainty over the possible locations via visual prompting. Given the current RGB image $I^t_c$, we first identify the free space seen in $I^t_c$ by (a) projecting it onto $M$, (b) keeping only the free points, and (c) sampling a set of points $P$ using farthest point sampling to ensure coverage. In practice, we use $| P | = 3$, which we find sufficient to cover the possible distinct regions in an image. Then, we de-project the sampled points back onto $I^t_c$ and annotate them with letters $\mathcal{Y}_P=\{\textrm`A\textrm', \textrm`B\textrm', \textrm`C\textrm'\}$ on $I^t_c$ to get an annotated image $I^t_{c, \mathcal{Y}_P}$, which can be used for visual prompting (\cref{fig:visual-prompting}) \cite{nasiriany2024pivot, liu2024moka}. Now, we have the following prompt:
\vspace{-8pt}
\begin{figure}[H]
\centering\includegraphics[width=0.98\linewidth]{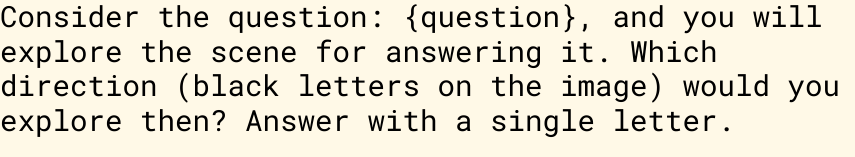}
\vspace{-15pt}
\end{figure}
We then use the (normalized) probability output of the VLM over each of the three directions to construct the \emph{local semantic value} (LSV) of $p \in P$:
\begin{equation}
    \text{LSV}_p(x^t) = \hat{f}_{y_\text{p}} (x^t) =  \hat{f}_{y_\text{p}}(I^t_c, s_{\text{LSV},q}) \in [0,1],
\end{equation}
where $x^t = (I^t_c, q)$ is the RGB image and question (cf. Sec.~\ref{sec:formulation}) and $s_{\text{LSV},q}$ is the prompt above with the question filled in. Note that this is a ``local'' score because the comparison is from one image, and the locations $P$ are not suited for being compared to those seen in images taken with different poses $g^t$ (\eg see top and bottom rows in \cref{fig:visual-prompting}) when planning the next robot pose using $M$. To address this issue, we additionally need to determine if we should navigate to poses visible from the current pose \emph{at all}. Similarly, we obtain the VLM's uncertainty 
via visual prompting:
\vspace{-8pt}
\begin{figure}[H]
\centering\includegraphics[width=0.98\linewidth]{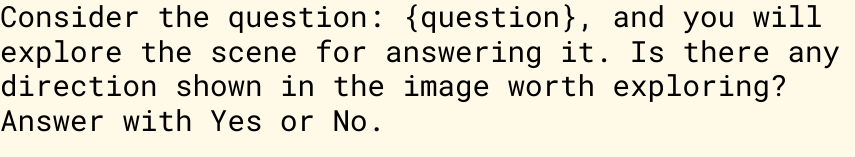}
\vspace{-15pt}
\end{figure}
This allows us to obtain the \emph{global semantic value} (GSV) of a given point $p \in P$ by querying the (normalized) probability of the VLM predicting `Yes':
\begin{equation}
    \text{GSV}_p(x^t) = \hat{f}_\text{`Yes'}(x^t) = \hat{f}_\text{`Yes'}(I^t_c, s_{\text{GSV},q})) \in [0,1],
\end{equation}
where again $s_{\text{GSV},q}$ is the prompt above with the question filled in. To determine the overall semantic value (SV), we apply temperature scaling ($\tau_\text{LSV}$ and $\tau_\text{GSV}$) to each of the two values and compute the following score:
\begin{equation}
        \text{SV}_p(x^t) = \exp \left(\tau_\text{LSV} \cdot \text{LSV}_p(x^t) + \tau_\text{GSV} \cdot \text{GSV}_p(x^t) \right).
\end{equation}
In practice, we apply Gaussian smoothing such that each value creates a Gaussian distribution around the point to better support the exploration strategy, which we explain next.

\subsection{Semantic-value-weighted Frontier Exploration}
\label{subsec:semantic-fbe}

Now we detail how to incorporate preferences in exploring high semantic-value regions using the semantic map --- we apply SV as the weights when sampling the next frontier to navigate to. Each weight are based on two values, $\text{SV}_p$, the semantic value at point $p$, and also $\text{SV}_{p, \text{Normal}}$, defined as the average semantic value of the points within a certain distance $d_\text{SV}$ from $p$ in the normal direction. $\text{SV}_{p, \text{Normal}}$ can be particularly useful to better guide the robot \emph{towards} the relevant regions if they are not close to robot's current pose. Gaussian smoothing around prompted points $P$ from \cref{subsec:semantic-value} also helps. We present the implementation details in App.~\cref{app:experiment-details}.

%% file: sections/approach-cp.tex
\section{Stopping Criterion for Exploration and Answering the Question}
\label{sec:stopping-criterion}

We have so far been using VLMs to guide the exploration for answering EQA. In the last section, we discussed how to address our first challenge of \emph{limited internal memory of VLMs} by building a semantic value weighted map and using it for efficient exploration. However, the second piece of efficient exploration is to know when you have enough information to answer the question and realize when you should stop exploring. 
This brings us to our second challenge of \emph{miscalibrated VLMs}, i.e., the fact that VLMs can be over-confident or under-confident about their answers.

Techniques for assessing VLM confidence in question answering typically rely on softmax scores. For example, one can compute the entropy of the predicted answer at each time step: 
\begin{equation}
\label{eq:entropy}
    H(\hat{f}(x^t)) = - \sum_{y \in \mathcal{Y}} \hat{f}_y(x^t) \log \hat{f}_y(x^t),
\end{equation}
and stop if this quantity is below a pre-defined threshold. 
Other techniques for assessing VLM confidence involve direct prompting\footnote{There is a subtle difference between this prompt and $s_{\text{GSV},q}$. This one is about \emph{answering the question} with the view, and $s_{\text{GSV},q}$ is for \emph{exploring possible directions} within the view.}:
\vspace{-8pt}
\begin{figure}[H]
\centering\includegraphics[width=0.98\linewidth]{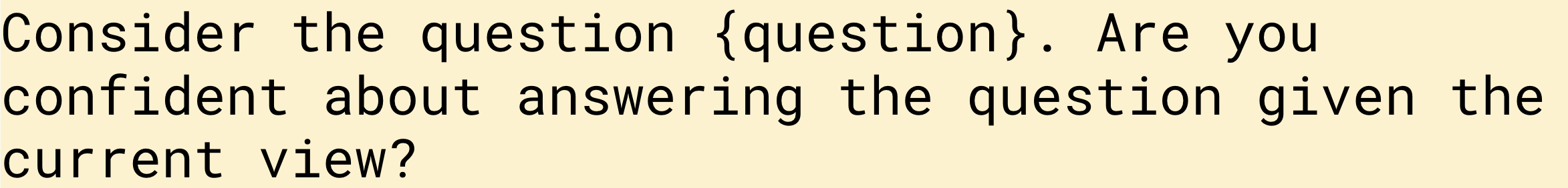}
\vspace{-15pt}
\end{figure}
We can then look at the probability of the model predicting `Yes' with this prompt; we refer to this as the \emph{question-image relevance score}:
\begin{equation}
\label{eq:relevance}
    \text{Rel}(x^t) = \hat{f}_\text{`Yes'}(I^t_c, (q,s_{\text{Rel},q})),
\end{equation}
where $s_{\text{Rel},q}$ is the prompt above with the question filled in. By normalizing this quantity with the sum of confidences of predicting `Yes' and `No', one obtains a scalar quantify bounded in $[0,1]$. A scalar threshold $h_\text{rel} \in [0,1]$ can then be used as the stopping criterion.


While these stopping criteria are simple to implement, relying on the raw softmax scores from the VLM faces a major challenge. The softmax scores from VLMs are often \emph{miscalibrated}, i.e., they are often over- or under-confident; this miscalibration is inherited from the underlying LLMs that are used to fine-tune VLMs~\cite{kadavath2022language}. Through experiments (\cref{sec:experiments}), we have found that the two options above with raw VLM softmax scores lead to the robot under-exploring or over-exploring in many scenarios. Also see App.~\cref{app:additional-results} for additional discussions on miscalibrated $\text{Rel}(x^t)$.

These observations motivate us to rigorously quantify the VLM's uncertainty and carefully calibrate the raw confidences. Our main insight is to employ multi-step conformal prediction, which allows the robot to maintain a \emph{set} of possible answers (\emph{prediction set}) over time, and stop when the set reduces to a single answer. Conformal prediction uses a moderately sized (e.g., $\sim$300) set of scenarios for carefully selecting a confidence threshold above which answers are included in the prediction set. This procedure allows us to achieve \emph{calibrated confidence}: with a user-specified probability, the prediction set is guaranteed to contain the correct answer for a new scenario (under the assumption that calibration and test scenarios are drawn from the same \emph{unknown} distribution $\mathcal{D}$). CP also minimizes the prediction set size \cite{vovk2005algorithmic, ren2023knowno}, which helps the robot to stop as quickly as it can while satisfying calibrated confidence.


\begin{figure}[h]
\centering\includegraphics[width=0.995\linewidth]{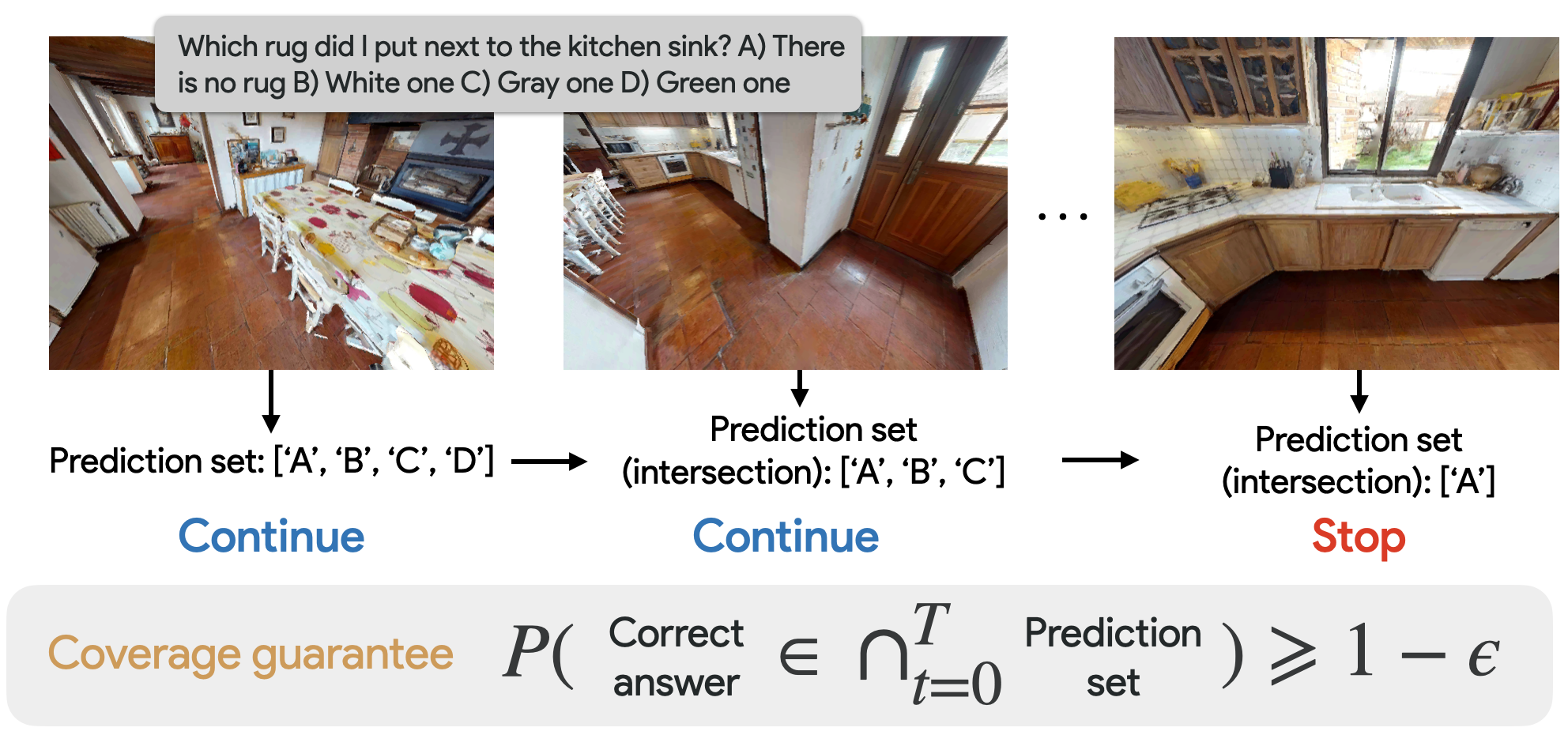}
    \vspace{-10pt}
    \caption{For determining when to stop, we apply a principled approach based on multiple-step conformal prediction: at each step, a prediction set is generated, and the robot keeps the intersection of the sets until there is only one option remaining. The correct answer is guaranteed to be the remaining one with user-specified probability.}
    \label{fig:cp}
    \vspace{-5pt}
\end{figure}

\subsection{Background: Conformal Prediction}
\label{sec:CP background}

We provide a brief overview of conformal prediction (CP) in this section; we refer the reader to~\cite{angelopoulos2023conformal} for a thorough exposition. We first describe a single-step setting where a VLM must answer a question pertaining to a \emph{single} image; we then describe CP in our multi-time-step active exploration setting in \cref{sec:multi-step CP}. 


Let $\mathcal{X}$ and $\mathcal{Y}$ denote the space of inputs (images and corresponding questions) and labels (answers) respectively, and let $\mathcal{D}$ denote an \emph{unknown} distribution over $\mathcal{Z}:= \mathcal{X} \times \mathcal{Y}$. Suppose we have collected a \emph{calibration dataset} $Z = \{z_i = (x_i, y_i)\}_{i=1}^N$ of such pairs drawn i.i.d. from $\mathcal{D}$. Now, given a new i.i.d. sample $z_\text{test} = (x_\text{test}, y_\text{test})$ with unknown true label $y_\text{test}$, CP generates a \emph{prediction set} $C(x_\text{test}) \subseteq \mathcal{Y}$ that contains $y_\text{test}$ with high probability \citep{vovk2005algorithmic}:
\begin{equation}
\label{eq:marginal_coverage}
    \mathbb{P}\big(y_\text{test} \in C(x_\text{test})\big) \geq 1-\epsilon.
\end{equation}
Here, $1-\epsilon$ is a user-defined threshold that impacts the size of $C(\cdot)$. 

CP provides this statistical guarantee on coverage by utilizing the dataset $Z$ to perform a calibration procedure with raw (heuristic) confidence scores. In our setting, we define the relevance-weighted confidence score for an input $x$ as:
\begin{equation}
\label{eq:rho}
\rho_y(x) := \text{Rel}(x)(\hat{f}_y(x)-1).
\end{equation}
This quantity is large when it is \emph{both} the case that the VLM is confident in the answer $y$ and the image is deemed highly relevant. CP utilizes these scores to evaluate the set of \emph{nonconformity scores} $\{\kappa_i = 1 - \rho_{y_i}(x_i)\}_{i=1}^N$ over the calibration set. Intuitively, the higher the nonconformity score is, the less confident the VLM is in the correct answer or the less relevant the image is deemed to be.
We then perform calibration by defining $\hat{q}$ to be the $\frac{\lceil (N+1)(1-\epsilon) \rceil}{N}$ empirical quantile of $\kappa_1, \dots, \kappa_N$. For a new input $x_\text{test}$, CP generates $C(x_\text{test}) = \{y \in \mathcal{Y} \ | \ \rho(x_\text{test})_y \geq 1-\hat{q})\}$, i.e., the prediction set that includes all labels in which the predictor has at least $1-\hat{q}$ relevance-weighted confidence. The generated prediction set ensures that the coverage guarantee in \cref{eq:marginal_coverage} holds.

\subsection{Applying Multi-Step CP for Embodied Question Answering}
\label{sec:multi-step CP}

Next, we describe how CP provides a \emph{principled} and more \emph{interpretable} stopping criterion for multi-step exploration by building on the multi-step CP approach presented in \cite{ren2023knowno}. In particular, we consider datapoints corresponding to episode-level \emph{sequences} of inputs. By performing calibration at the sequence level using a carefully chosen non-conformity score function, we can ensure that prediction sets can be constructed causally (i.e., in the order of time-step) at test time. 

Let $x^t$ denote the input at time $t$ consisting of the RGB image $I_c^t$ and the question $q$. Each episode results in a sequence $\bar{x} = (x^0, x^1, \dots)$ of such inputs. The distribution $\mathcal{D}$ over scenarios (cf. \cref{sec:formulation}) along with the exploration policy induces a distribution over input sequences $\bar{x}$. We first define the relevance-weighted confidence score at time $t$ (analogous to the single-step definition Eq. \eqref{eq:rho}):
\begin{equation}
\rho^t_y(x^t) := \text{Rel}(x^t)(\hat{f}_y(x^t)-1).
\end{equation}
This quantity is large when the input $x^t$ at time $t$ is deemed highly relevant and the VLM is confident in the answer $y$. We can then define the \emph{episode-level} confidence as:
\begin{equation}
\bar{\rho}_y(\bar{x}) := \min_{t \in [T]} \  \rho^t_y(x^t),
\end{equation}
where $T$ is the maximum allowable episode length.
Given a calibration dataset $Z = \{z_i = (\bar{x}_i, y_i)\}_{i=1}^N$ of input sequences (collected using the exploration policy) and ground-truth answers, we define the non-conformity score for data point $i$ as $\kappa_i := 1 - \bar{\rho}_{y_i}(\bar{x}_i)$. 

We can then perform the standard CP calibration as described in~\cref{sec:CP background} using these non-conformity scores in order to obtain a confidence threshold $\hat{q}$. Then, given a new input sequence $\bar{x}_\text{test}$, we can construct a \emph{sequence-level} prediction set $\bar{C}(\bar{x}_\text{test}) := \{y \in \mathcal{Y} | \bar{\rho}_y(\bar{x}_\text{test}) \geq 1 - \hat{q}\}$. This set is guaranteed to contain the ground-truth answer with probability $1-\epsilon$. 

However, at test-time, the robot does not obtain the entire sequence $\bar{x}_\text{test}$ at once; instead, the prediction sets must be \emph{causally} constructed over time (i.e., using observations up to the current time). Define the causally constructed prediction set at time $t$ to be:
\begin{equation}
C^t(x^t_\text{test}) := \{y \in \mathcal{Y} | \rho^t_y(x^t_\text{test}) \geq 1 - \hat{q}\}.
\end{equation}

\begin{claim}
    For all time $t \in [T]$, the causally constructed prediction set $C^t(x^t_\text{test})$ contains the sequence-level set $\bar{C}(\bar{x}_\text{test})$. Moreover, $\cap_{t=0}^T C^t(x^t_\text{test}) = \bar{C}(\bar{x}_\text{test})$. 
\end{claim}
\begin{proof}
    See App.~\cref{app:proof}.
\end{proof}

\begin{proposition}
    With probability $1-\epsilon$ for test scenarios drawn from $\mathcal{D}$, the ground-truth label $y_\text{test}$ is contained in the prediction set $\cap_{k=0}^t C^k(x^k_\text{test})$ for all $t \in [T]$. 
\end{proposition}
\begin{proof}
    This follows directly from the claim above and the fact that the sequence-level prediction set $\bar{C}(\bar{x}_\text{test})$ contains the ground-truth label with user-defined probability $1-\epsilon$ as guaranteed by CP.
\end{proof}

At test time, we thus construct the set $C^t(x^t_\text{test})$ at each step and maintain the intersection of these sets over time. If the resulting intersection contains only a single element, the robot halts its exploration with $1-\epsilon$ confidence that the corresponding answer is correct \textcolor{blue}{(\cref{fig:cp})}. Alternately, if the maximum allowable time horizon $T$ is reached and the intersected set contains multiple answers, or the intersected set is empty, the robot returns the answer $y$ with highest $\hat{f}_y(x^t)$ from time $t$ with the highest $\text{Rel}(x^t)$.

%% file: sections/dataset.tex
\section{HM-EQA Dataset}
\label{sec:dataset}

\begin{figure}[h]
\centering\includegraphics[width=0.995\linewidth]{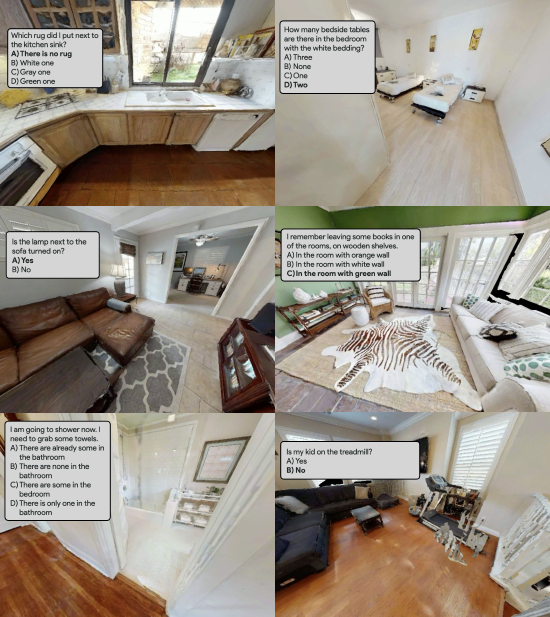}
    \vspace{-10pt}
    \caption{Sample scenarios from the HM-EQA dataset. The images are the views used by the robot to determine the final answer in our experiments. The boxes show the questions with the true answers bolded.}
    \label{fig:scenes}
    \vspace{-5pt}
\end{figure}

While prior work has primarily considered synthetic scenes and simple questions such as ``what is the color of the coffee table?'' or ``how many sofas are there in the living room?'' involving basic attributes of relatively large pieces of furniture, we are interested in applying our VLM-based framework in more realistic and diverse scenarios, where the question can be more open-ended and possibly require semantic reasoning. To this end, we propose HM-EQA, a new EQA dataset based on the Habitat-Matterport 3D Research Dataset (HM3D), which provides hundreds of photo-realistic, diverse indoor 3D scans~\cite{ramakrishnan2021habitat}. Sample scenes are shown in \cref{fig:scenes}.

In order to generate questions that are realistic in typical household settings, we leverage GPT4-V, the state-of-the-art VLM, to generate such questions based on twelve random views sampled inside an indoor scene from HM3D, and also three sets of example manually-written questions and answers given views of the corresponding scenes (one set per scene, see details in App.~\cref{app:experiment-details}). Afterwards we manually remove some of the questions that are (1) too simple (\eg ``How many sofas are there in the living room for them to sit on?") or (2) hallucinating objects that cannot be seen from the views by a human (\eg eyeglasses, watering can, and remote control). We consider option (1) to be too simple as it involves detection of very prominent objects in the scene (large in size). At the end, we generate 500 questions from 267 different scenes. The resulting questions can be roughly divided into five categories (also showing their split within the whole dataset):
\begin{enumerate} 
    \item \textbf{Identification (16.2$\boldsymbol\%$):} asking about identifying the type of an object, \eg ``Which tablecloth is on the dining table? A) Red B) White C) Black D) Gray''  
    \item \textbf{Counting (16.0$\boldsymbol\%$):} asking about the number of objects, \eg ``My friends and I were playing pool last night. Did we leave any cues on the table? A) None B) One C) Two D) Three''.   
    \item \textbf{Existence (21.6$\boldsymbol\%$):} asking if an object is present at a location, \eg ``Did I leave my jacket on the bench near the front door? A) Yes B) No''.      
    \item \textbf{State (20.2$\boldsymbol\%$):} asking about the state of an object, \eg ``Is the air conditioning in the living room turned on? A) Yes B) No'' or ``Is the curtain in the master bedroom closed? A) Yes B) No''. 
    \item \textbf{Location (26.0$\boldsymbol\%$):} asking about the location of an object, \eg ``Where have I left the black suitcase? A) At the corner of the bedroom B) In the hallway C) In the storage room D) Next to TV in the living room''.  
\end{enumerate}

Notice that some of the questions only involve two multiple choices, and our formulation in \cref{sec:formulation} assumes four. For consistency, if the question itself does not have four multiple choices, we add additional ones, \eg ``D) (Do not choose this option)'' until there are four.

Since the different scenes $e$ from HM3D can have very different sizes (majority of which range from \SI{100}{\metre\squared} to \SI{800}{\metre\squared}), we set the maximum allowed time steps $T_\epsilon$ in each scene to be the square root of the 2D size times a factor of three. The initial pose of the robot $g^0$ is sampled randomly from the free space in the scene. We have now fully defined the scenarios introduced in \cref{sec:formulation}, $\epsilon := (e, T, g^0, q, y)$ ($q$ for question and $y$ for answer).

\subsection{Extension: Answering Questions without Multiple Choices}
\label{seubsec: no-mc}
While the HM-EQA dataset contains multiple-choice questions instead of ones without multiple choices and we believe this is a reasonable design choice for benchmarking purposes, our framework can be also be extended to handle questions without multiple choice answers. We can apply the multiple-choice question answering (MCQA) setup in \cite{ren2023knowno} where an LLM (VLM in our case) first proposes different possible answers given the input (image observation and question in our case) based on few-shot examples, and then chooses one among them. To ensure coverage of the ground truth answer, we will add an additional option `An option not listed here' in case the VLM does not generate the correct one. This setup allows us to quantify VLM's uncertainty among different possible answers in \cref{sec:multi-step CP}. \cref{fig:fewshot} shows an example of GPT4-V generating possible answers to a question given the view.

\begin{figure}[h]
\centering\includegraphics[width=0.995\linewidth]{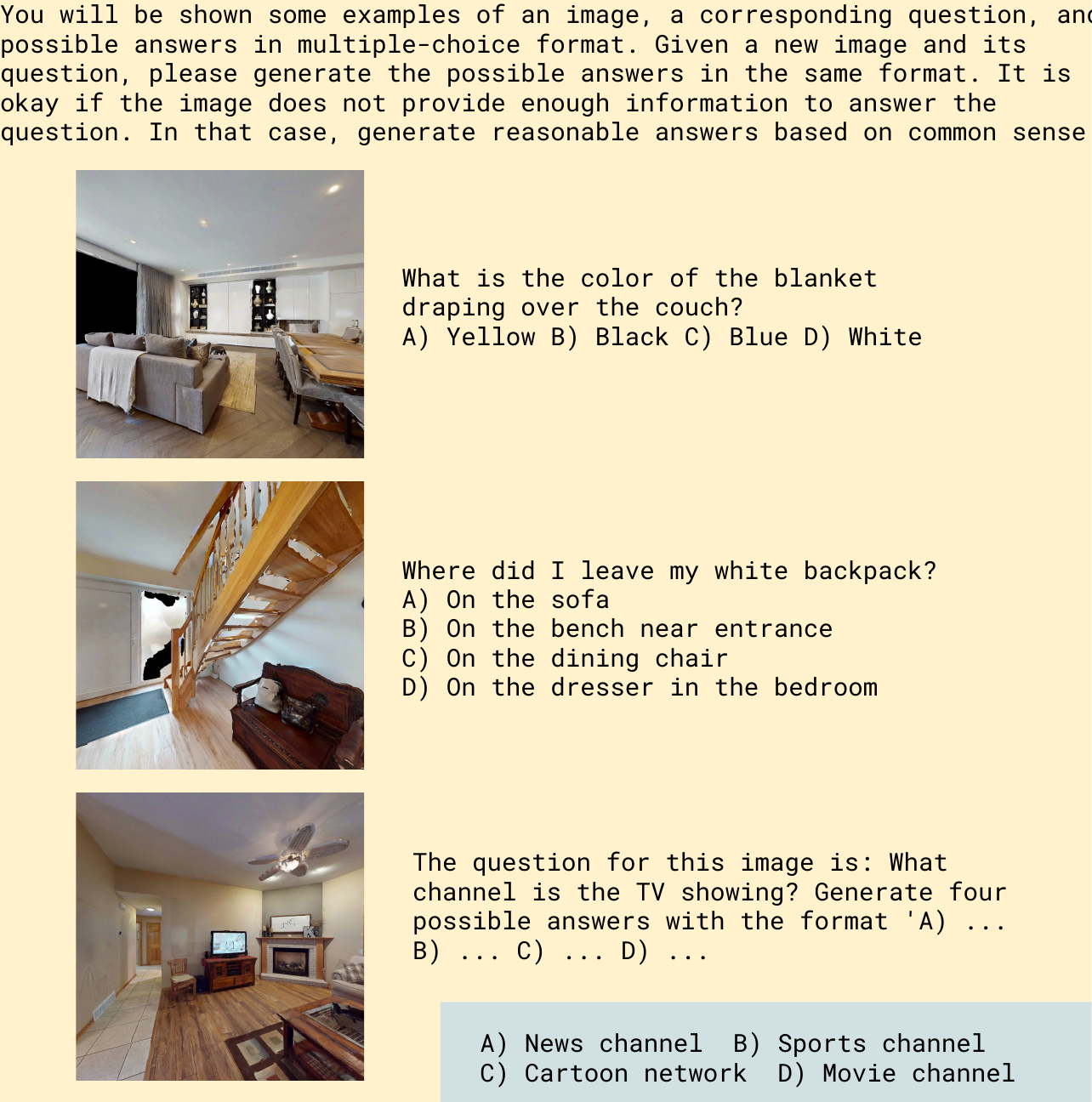}
    \vspace{-10pt}
    \caption{Prompt (yellow) and output (blue) of few-shot prompting GPT4-V to generate possible answers to the question given the view. The correct answer is `C) Cartoon network'. The actual prompt used has images encoded into a text form.}
    \label{fig:fewshot}
    \vspace{-5pt}
\end{figure}

%% file: sections/experiment.tex
\section{Experiments and Discussion}
\label{sec:experiments}

\begin{figure*}[t!]
\begin{center}
\includegraphics[width=0.95\textwidth]{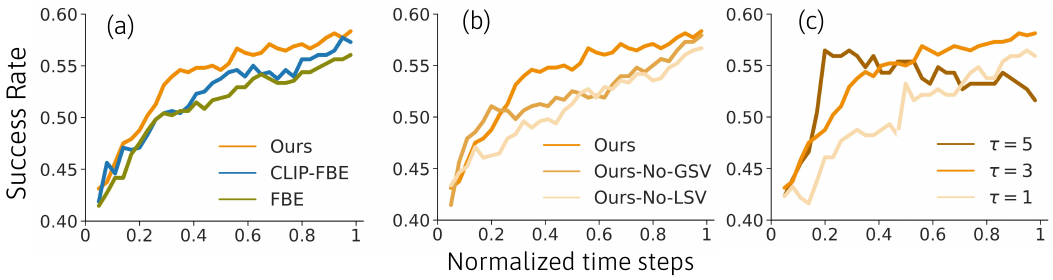}
\caption{Normalized time step taken vs. success rate in simulated experiments for (a) comparing different exploration methods, (b) ablating GSV and LSV from our method, and (c) varying semantic value temperature scaling. Ours improves exploration efficiency leveraging VLM reasoning.}
\label{fig:exploration-comparison}
\vspace{-15pt}
\end{center}
\end{figure*}

Through extensive simulated and hardware experiments, we investigate the following key questions:
\begin{itemize}[leftmargin=*]
    \item Q1 - \textbf{Semantic exploration:} Does our method use fewer steps while achieving the same level of EQA success compared to baselines without leveraging VLM reasoning for exploration?

    \item Q2 - \textbf{Stopping criteria:} Does our method with confidence calibration use fewer steps while achieving the same level of EQA success compared to baselines without calibration? 
\end{itemize}

\subsection{Implementation Details}
We use Prismatic VLM \cite{karamcheti2024prismatic}, a recently trained VLM that exhibits strong question answering and spatial reasoning capabilities among open-source models. Since our setup requires probability outputs from the VLM, we cannot use the state-of-the-art VLM, GPT4-V. For simulated experiments, we run the scenarios from HM-EQA in the Habitat simulator \cite{szot2021habitat}.

For hardware experiments, we deploy our framework with a Fetch mobile robot \cite{wise2016fetch} in six scenarios (three different home-office-like scenes (\cref{fig:fetch-comparison}) and two sets of question/answer for each scene). An iPhone 12 is mounted on top of the robot to acquire wide-angle RGB images. 

\subsection{Q1: Semantic Exploration - Baselines}
To evaluate how well our VLM-based semantic exploration performs, we consider the following baselines:
\begin{itemize}[leftmargin=*]
    \item \textbf{FBE}: This baseline applies frontier-based exploration without using any semantic values for weighted sampling.
    \item \textbf{CLIP-FBE}: This baseline considers semantic values when sampling the next frontier like ours. However, the semantic values are from CLIP \cite{radford2021learning}, which provides the relevance of an image given a text input. 
    CLIP-based methods have been very popular in tasks like zero-shot object navigation \cite{huang2023visual, gadre2022clip}. We apply the multi-scale relevancy extractor from \cite{ha2022semabs} to extract a dense pixel-wise CLIP score to allow fine-grained localization of relevant objects or exploration directions. We consider this a strong baseline.
    \item \textbf{Ours-No-LSV}: This baseline is the same as our method but without the Local Semantic Value (LSV). Only relying on the Global Semantic Value (GSV) means less fine-grained semantic values at the semantic map.
    \item \textbf{Ours-No-GSV}: This baseline is the same as our method except it does not use GSV. Only relying on LSV means that the planner could be myopic and may lead the robot to explore overall less relevant regions.
\end{itemize}

For evaluation, we vary the number of time steps (normalized by $T_\epsilon$ of each scenario, defined in \cref{sec:dataset}) allowed --- the robot does not need to determine when to stop in these experiments (we will consider the different stopping criteria in Q2; Sec.~\ref{sec:stopping criterion experiments}). For all methods including ours, in each scenario the robot stops at the maximum number of time steps allowed, and the final answer is chosen to be $y$ with highest $\hat{f}_y(x^t)$, at time $t$ with the highest $\text{Rel}(x^t)$. Also, since we do not use conformal prediction as a stopping criterion for these experiments, we do not need to split the data into calibration and test scenarios. We can thus use all 500 scenarios from the dataset for evaluation. 

\subsection{Q1: Semantic Exploration - Simulation Results}

\cref{fig:exploration-comparison}a shows the average success rate (answer correctly predicted) achieved when the robot is allowed to run for different numbers of time steps. Compared to FBE and CLIP-FBE, our method uses fewer steps to achieve success. The difference becomes significant around $20\%-30\%$ of the maximum allowable steps, demonstrating the effect of VLM-based active exploration at the early stage of the episodes. We also present the results for each of the five question categories (defined in \cref{sec:dataset}) in App.~\cref{app:additional-results}.

In \cref{fig:exploration-comparison}b, we demonstrate that both GSV and LSV are critical to achieving high success rates when using semantic values to guide exploration. Without GSV or LSV, the performance of our method is on par with FBE or even worse. Especially, without GSV, the plot shows the success rate is actually higher than ours with GSV at the early stage, but improves less afterwards --- this indicates the planner being myopic as it only considers LSV (obtained within a single view), and thus the robot being stuck in incorrectly chosen locations and unable to explore other locations. Without LSV, the robot explores less efficiently overall due to the less fine-grained semantic values.

In \cref{fig:exploration-comparison}c, we vary the temperature scaling $\tau_\text{LSV}$ and $\tau_\text{GSV}$ applied when determing the semantic values (SV) used for sampling the frontier. The higher the scalings are, the bigger the difference in SV among different regions, leading to a higher degree of semantic exploration. We set $\tau=1,3,5$ for both $\tau_\text{LSV}$ and $\tau_\text{GSV}$. Results show that too high $\tau$ leads to faster exploration at the beginning, but worse performance in later normalized time steps. This is potentially due to the robot overly prioritizing the semantic regions --- sometimes the VLM's reasoning can be less reliable. Too low $\tau$ also leads to inferior efficiency. Future work can explore adaptively determining the scaling within an episode to better balance semantic exploration and pure-frontier-based exploration.

We also would like to point out that the overall relatively low success rate (around $60\%$ with maximum time steps), is largely due to the wrong VLM answer prediction even when the robot sees the relevant object and the VLM deems the images relevant. 
We envision that these cases can be reduced with improved QA capabilities of VLMs in the near future. \cref{fig:fetch-comparison} also shows failure cases of the VLM question answering in hardware experiments (Scenario 5 and 6).

\subsection{Q1: Semantic Exploration - Comparing to CLIP-FBE} Although CLIP has shown great promises in helping robot find objects in zero-shot \cite{gadre2022clip} and CLIP-FBE shows improved exploration efficiency in \cref{fig:exploration-comparison}, it also has the drawback of (1)~behaving like bag-of-words \cite{yuksekgonul2022and}, and (2) thus offering limited semantic reasoning. \cref{fig:compare-clip} shows such an example: the question is ``Is the gray curtain pulled down in the bedroom?''. The current view shows a partial view of a living room and a door to the bedroom with a glimpse of the nightstand. When we prompt CLIP with the full question (bottom left), there is no visible attention on the nightstand but rather on the sofa. Even when we change the prompt to only the word ``Bedroom'', the attention is still minimal. Instead, when we prompt the VLM (top), it is able to offer exploring through the door for finding the bedroom. We believe the VLM's capacity for such semantic reasoning (\ie seeing the doors as a potential path to the bedroom, given that the current room is a living room) helps with exploration.

\begin{figure}[h]
\centering\includegraphics[width=0.995\linewidth]{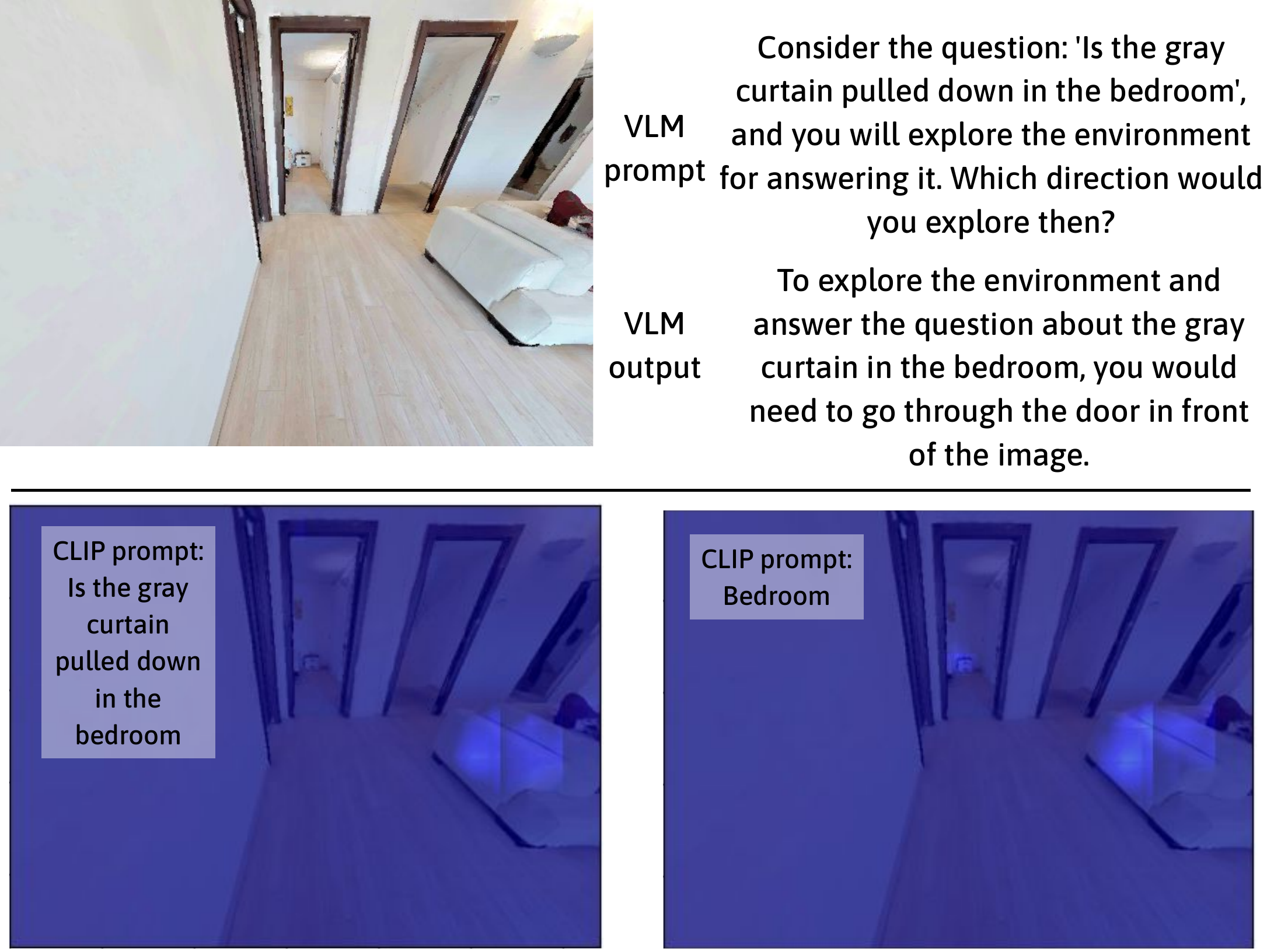}
    \vspace{-5pt}
    \caption{Top: Image, prompt, and output of the VLM. It can reason about exploring certain directions, in this case, through the door leading to the bedroom. Bottom: Multi-scale CLIP relevancy \cite{ha2022semabs} of the full image. CLIP may fail to offer exploration signals when the view shows little information about the place of interest.}
    \label{fig:compare-clip}
    \vspace{-12pt}
\end{figure}

Nonetheless, we find there are other cases where the VLM can be over-confident and leads the robot to over-prioritize a certain area that turns out to be irrelevant (leading to worse performance when a very high degree of semantic exploration is applied, as shown in \cref{fig:exploration-comparison}c). We hope future VLMs can be trained or fine-tuned to exhibit stronger exploration capabilities with calibrated confidence.

\subsection{Q1: Semantic Exploration - Fine-tuning VLM}
While the main experiments do not use fine-tuned VLMs — since the base performance of the VLM is not a contribution of this work — fine-tuning VLMs with in-domain examples can improve the results and alleviate reliance on performance of existing VLMs. In an additional experiment, we split the 500 scenarios into 300 for training and 200 for testing, and use the 300 training ones for fine-tuning the VLM specifically for question answering. We find, using the same Prismatic VLM, the success rate improves to 68.1\% from 56.2\%, indicating that fine-tuning on the dataset can better ground VLM’s question answering in the scenario distribution and EQA tasks.

\subsection{Q2: Stopping Criterion - Baselines}
\label{sec:stopping criterion experiments}
To evaluate how well our CP-based stopping criterion performs, we consider the following baselines:
\begin{itemize}[leftmargin=*]
    \item \textbf{Entropy}: This baseline uses the entropy of the predicted answer \eqref{eq:entropy} as the metric, and stops exploration once it is lower than some threshold (which can be varied). The final answer is the one with the highest $\hat{f}_y(x^t)$ at the stopping time step.
    \item \textbf{Relevance}: use the question-image relevance score \eqref{eq:relevance} as the metric, and stop exploration once it is higher than some threshold (which can be varied). The final answer is the one with the highest $\hat{f}_y(x^t)$ at the stopping time step.
\end{itemize}

For our method, we split the 500 scenarios from the dataset randomly into 300 for calibration and 200 for testing. We evaluate all methods using the same set of test scenarios.

\subsection{Q2: Stopping Criterion - Simulation Results}
For evaluation, we vary the different thresholds used by our method ($\epsilon$), Entropy, and Relevance. We then consider the normalized time step vs. the achieved success rate.

\cref{fig:stop-comparison} shows that our method significantly outperforms Entropy, which relies only on the answer prediction \cref{eq:entropy}. We find that, as the robot often sees irrelevant views (\eg facing an empty wall), the VLM still outputs highly confident, biased answers for the question. Such bias leads to low prediction entropy and the robot stops prematurely.

\begin{figure}[h]
\centering\includegraphics[width=0.75\linewidth]{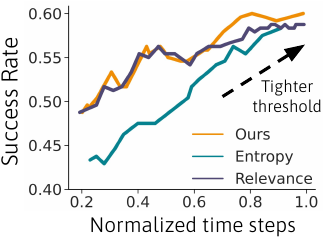}
    \vspace{-2pt}
    \caption{Normalized time step taken vs. success rate using varying threshold in Ours and the baselines, in simulated experiments. $\epsilon$ is varied from $0.27$ to $0.99$ (large value means higher confidence), question answering entropy threshold from $0.08$ to $0.001$ (smaller value means higher confidence), and question-image relevance threshold from $0.1$ to $0.99$ (large value means higher confidence).}
    \label{fig:stop-comparison}
    \vspace{-12pt}
\end{figure}

This observation leads to the necessity of using the question-answer relevance score \cref{eq:relevance}, which helps the robot ignore some of the irrelevant views and continue exploring. However, we find that, in order to achieve high success rates (upper right side of the plot), Relevance tends to use more time steps. For example, to achieve $58\%$ success rate, our method takes about 71$\%$ of the maximum time steps while Relevance takes 85$\%$. This indicates that without calibrating the question-relevance score and the answer prediction, the baselines may over-explore. In contrast, our method, based on the theory of multi-step conformal prediction \cref{sec:multi-step CP}, calibrates the VLM's confidence and consequently improves exploration efficiency. Our results also corroborate the findings from \cite{ren2023knowno} that CP offers the most efficiency improvement when the desired success rate is high. 


\subsection{Q2: Stopping Criterion - Hardware Results} We focus on comparing the stopping criterion performance in the hardware experiments. First we determine the threshold used for the stopping criterion: $\epsilon=0.5$ for the threshold used in CP calibration for our method, 0.4 for Relevance, and 0.1 for Entropy. These thresholds roughly corresponds to $50\%$ success rates based on simulated experiment results from \cref{fig:stop-comparison}.

\begin{figure*}[t!]
\begin{center}
\includegraphics[width=0.99\textwidth]{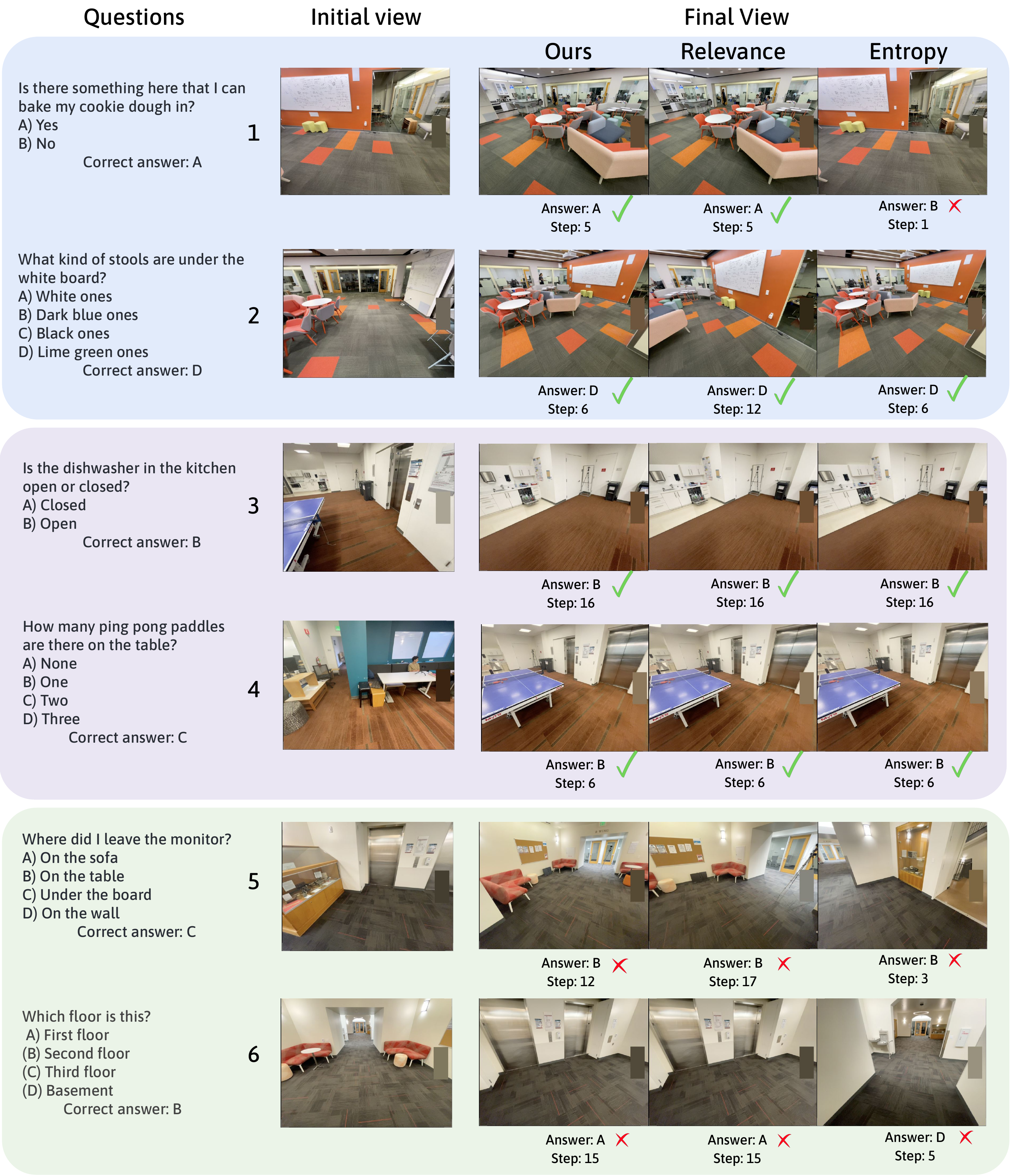}
\caption{Six scenarios considered in hardware experiments. We show the results for question answering and stopping steps for our method vs. the two baselines. Ours achieves the best success rate while using fewer steps to stop.}
\label{fig:fetch-comparison}
\vspace{-12pt}
\end{center}
\end{figure*}

\cref{fig:fetch-comparison} shows all six scenarios (questions/answers, and the initial robot views), the final views after the robot stops exploration using different methods (Ours, Relevance, and Entropy), the final answers chosen, and the number of steps taken at stopping (not normalized). In Scenario 3 and 4, all methods stop at the same time step and answer the question correctly. Looking at other scenarios, Entropy tends to stop early (\eg Step 1 in Scenario 1 and Step 3 in Scenario 5), but this leads to the failure in Scenario 1 where the other two methods answer correctly based on later views. Relevance achieves the same success rate (4 out of 6) as our method, but it uses more steps in Scenario 2 and Scenario 5, while our approach answers the question based on relevant views from previous steps --- in Scenario 2, Ours decides to stop after seeing the lime green stools at Step 6, and in Scenario 5, Ours decides to stop after seeing the monitor under the board at Step 12. Overall, our method achieves the best success rate (same as Relevance) and improves the efficiency.

We again note that the two failed scenarios, 5 and 6, are mostly due to the incorrect prediction of the VLM even when seeing the relevant views. In Scenario 5, the views from Ours and Relevance show the monitor under the board on the ground, and in Scenario 6, the views show the elevators, which have the sign ``2'' indicating the second floor. However, the VLM answer predictions are both wrong. 


%% file: sections/conclusion.tex
\section{Conclusion}

To better solve Embodied Question Answering tasks, we propose a framework that leverages VLM's commonsense reasoning for exploring relevant locations, as well as calibrating its confidence such that the robot stops exploration at the right moment. Our extensive simulated and hardware experiments show that our method improves the exploration efficiency (achieving similar success levels using fewer steps). We believe that the performance of our method will keep improving with the advent of more powerful VLMs with stronger spatial and semantic reasoning while our method also complements them.

\textbf{Limitation and Future Work.} Our current setup does not consider the view orientation when considering the semantic value, which can be important for question answering and exploration. Our performance is also limited by the spatial reasoning capabilities of current open-source VLMs being inconsistent over scenarios.
One promising direction to address this is to collect such exploration data using simulated scenes (\eg with our HM-EQA dataset) or from egocentric human data (\eg Ego4D dataset \cite{grauman2022ego4d}), and train or fine-tune the VLM to predict possible exploration directions or reason about them. Another useful direction is to incorporate multiple views as input to the VLM --- keeping all views seen so far and effectively retrieving relevant ones for reasoning at the current step can be very useful. On the calibration side, one exciting direction is to use the prediction set from CP not only as the stopping criterion, but also for \emph{guiding exploration} --- \eg the set contains possible locations of an object to be found.


%% file: sections/appendix.tex
\section{Proof of Claim 1}
\label{app:proof}
\begin{proof}
The proof follows \citet[Claim 1]{ren2023knowno}. Specifically, 
\begin{align}
    y \in \bar{C}(\bar{x}_\text{test})
    & \iff \bar{\rho}_y(\bar{x}_\text{test}) \geq 1 - \hat{q} \\
    & \iff \min_{t \in [T]} \  \rho^t_y(x^t_\text{test}) \geq 1 - \hat{q} \\
    & \iff \rho^t_y(x^t_\text{test}) \geq 1 - \hat{q}, \ \forall t \in [T] \\
    & \iff y \in C^t(x^t_\text{test}), \ \forall t \in [T] \\
    & \iff y \in \cap_{t=0}^T \ C^t(x^t_\text{test}).
\end{align}
\end{proof}

\section{Additional implementation details}
\label{app:experiment-details}

\subsection{Semantic map}

\smallskip \noindent \textbf{Updating the explored regions.} As introduced in  \cref{subsec:map}, while all voxels seen in the depth image $I^t_d$ are used to update occupancy at each step, only those within a smaller field of view are used to update whether they have been explored, enabling  more fine-grained exploration. In practice, we use a $4:3$ aspect ratio for the images, and $120$ degrees for the horizontal field of view (HFOV) and $105$ degrees for the vertical field of view (VFOV) in simulation. Then we mark voxels explored if they correspond to pixels from the middle $50\%$ of the full HFOV and the lower $50\%$ of the full VFOV, as these pixels correspond to voxels closer to the robot.

\smallskip \noindent \textbf{Determining the weights of frontiers based on semantic values.} From \cref{subsec:semantic-value}, when the robot plans for the next pose to travel to, it samples from possible frontiers with weights, and the weight of each frontier depends on (1) $\text{SV}_p$, the semantic value at the frontier $p$, and (2) $\text{SV}_{p,\text{normal}}$, the average semantic value of the points with $d_\text{SV}$ distance from $p$ in the normal direction. \cref{fig:semantic-value} illustrates the setup. We set $d_\text{SV} = 3\si{m}$. Notice that the frontier near the top of the figure has a slightly higher $\text{SV}$, while the middle frontiers have much higher $\text{SV}_{p,\text{normal}}$ due to the high semantic value region at about $2\si{m}$ away from the frontiers into the un-explored regions. Balancing between $\text{SV}$ and $\text{SV}_{p,\text{normal}}$ as the sampling weights helps the robot explore and move towards relevant regions.

\begin{figure}[h]
\centering\includegraphics[width=0.99\linewidth]{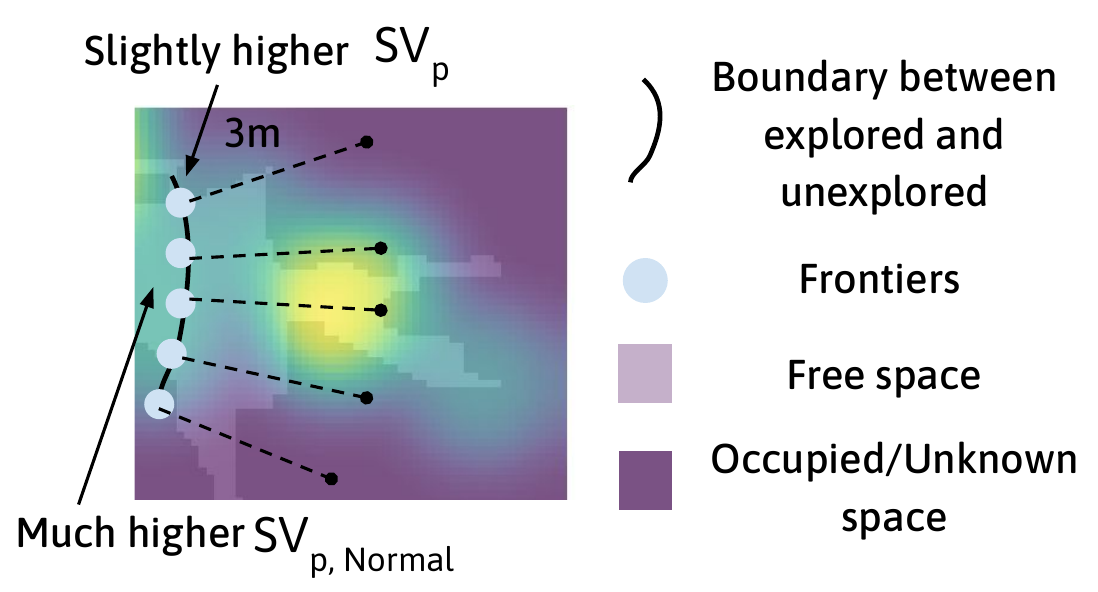}
    \vspace{-8pt}
    \caption{Sampling weight of the frontier $p$ depends on both $\text{SV}_p$ and $\text{SV}_{p,\text{normal}}$.}
    \label{fig:semantic-value}
    \vspace{-5pt}
\end{figure}

Similar to SV combining LSV and GSV in \cref{subsec:semantic-value}, we again apply temperature scaling ($\tau_\text{SV}$ and $\tau_\text{SV, Normal}$) to each of the two values and compute the final weight $w_p$ of the frontier $p$:
\begin{equation}
    \text{w}_p = \exp \left(\tau_\text{SV} \cdot \text{SV}_p + \tau_\text{SV, Normal} \cdot \text{SV}_{p,\text{normal}} \right).
\end{equation}
In practice we use $1$ for both scaling values. Online adaptation of these values can potentially further improve the exploration efficiency. 


\subsection{HM-EQA Dataset}
\smallskip \noindent \textbf{Generating candidate questions and answers with GPT4-V.} We leverage GPT4-V to help generate candidate questions and their answers given multiple views of each scene. GPT4-V is given the instructions with the prompt shown in \cref{fig:gpt4v-prompt}, as well as a set of twelve views randomly sampled inside the scene (\cref{fig:gpt4v-views}). We also include three examples of the questions and answers (manually written) to help GPT4-V generated desired ones (questions and answers shown in \cref{fig:few-shot}, views skipped).

\begin{figure}[h]
\centering\includegraphics[width=0.99\linewidth]{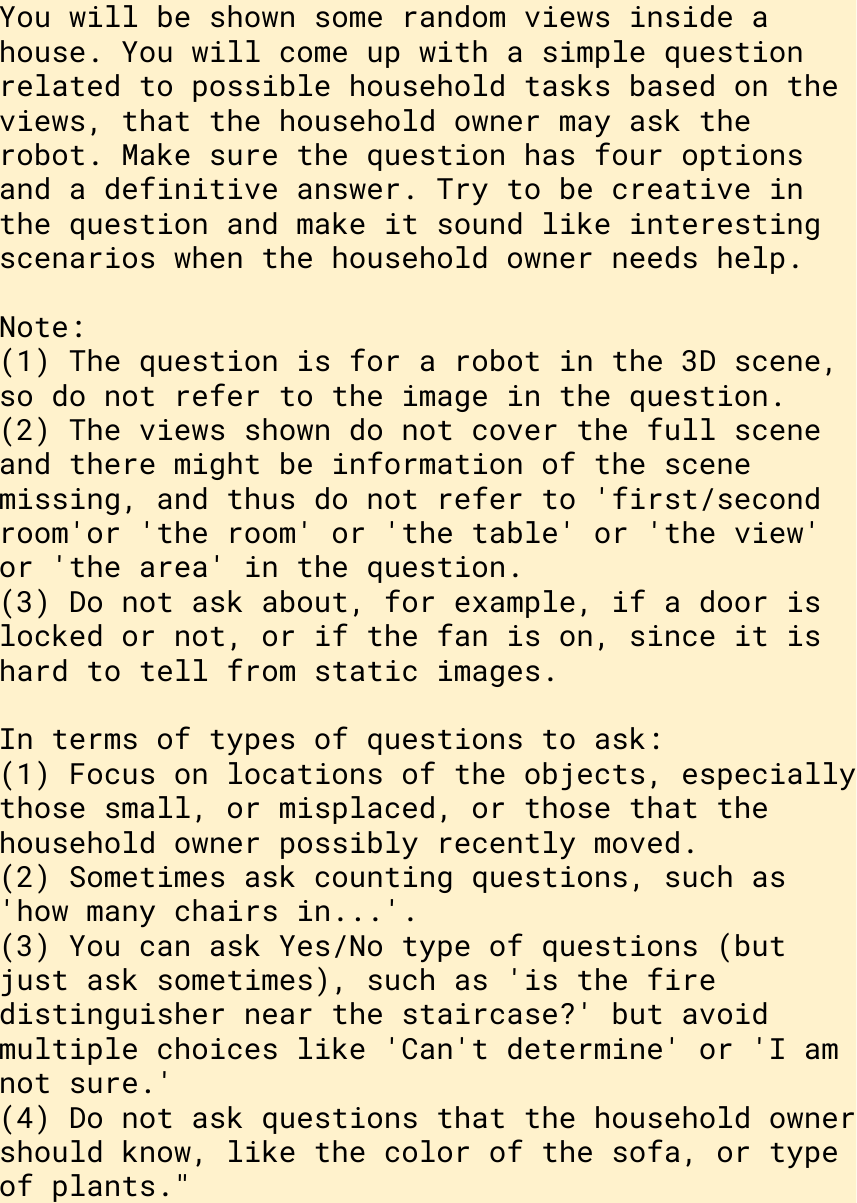}
    \vspace{-8pt}
    \caption{Prompt used when GPT4-V generates candidate questions and their answers given multiple views of the scene.}
    \label{fig:gpt4v-prompt}
    \vspace{-5pt}
\end{figure}

\begin{figure}[h]
\centering\includegraphics[width=0.99\linewidth]{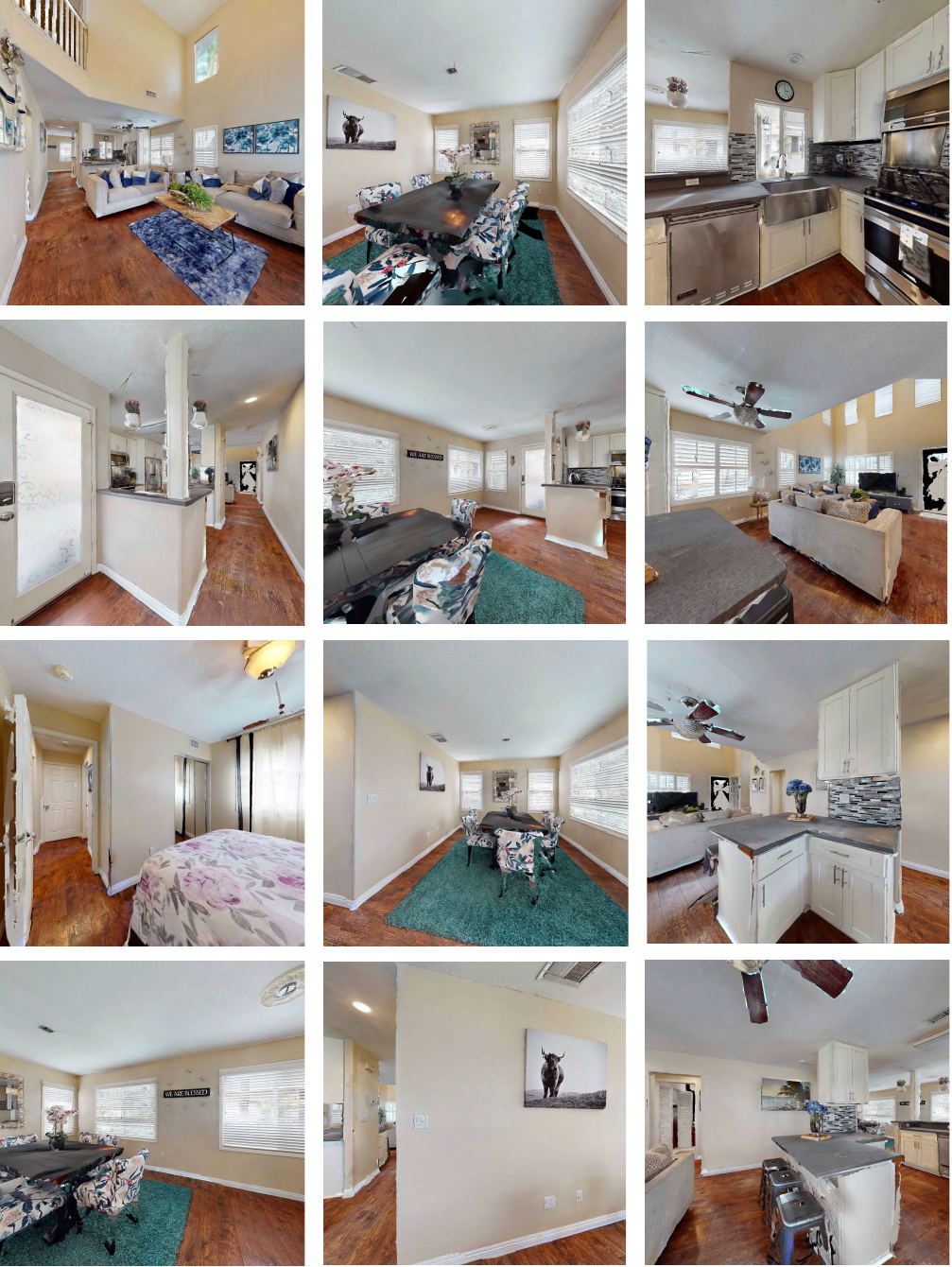}
    \vspace{-8pt}
    \caption{Twelve random views sampled within the scene for prompting GPT4-V to generate candidate questions and their answers. In this case, a question about the third image (first row) is generated by GPT4-V: `where is the wall clock placed in the kitchen? A) Above the sink B) On the refrigerator C) On the cabinet D) Above the oven Answer: A) Above the sink'.}
    \label{fig:gpt4v-views}
    \vspace{-5pt}
\end{figure}

\begin{figure}[h]
\centering\includegraphics[width=0.99\linewidth]{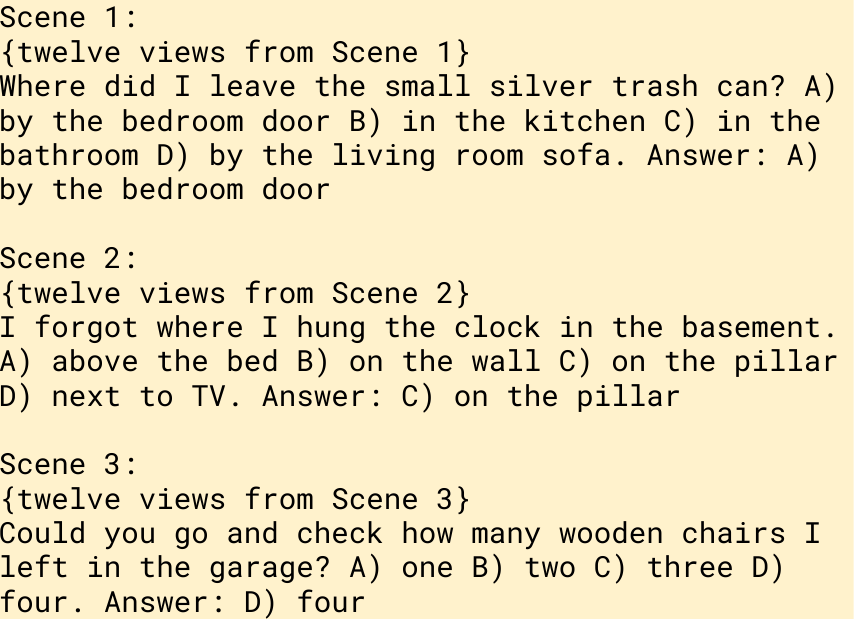}
    \vspace{-8pt}
    \caption{Three example question and answer pairs for GPT4-V. Actual views of the scenes not shown.}
    \label{fig:few-shot}
    \vspace{-5pt}
\end{figure}


\section{Additional experiment results}
\label{app:additional-results}

\smallskip \noindent \textbf{EQA results using other VLMs.} The exploration and question answering performance of our method scales with the capabilities of the VLMs. We believe that VLMs have improved significantly and will keep improving over time. For example, there have been new advances in open-source VLMs such as LLaVA 1.6 improving significantly across benchmarks compared to previous ones. We run our exploration experiments with LLaVA 1.6 without fine-tuning, and we found the success rate using the maximum normalized time step improves to 65.3\% from 58.4\% (\cref{fig:exploration-comparison} left) achieved in the main experiments. In addition, the official version of state-of-the-art GPT4-V provides access to the log probabilities of its output, which was not available with the preview version. We run our exploration experiments with GPT4-V without fine-tuning, and we found the success rate improves to 73.9\% from 58.4\%. These results highlight the better question answering and exploration capabilities of the VLMs. Nonetheless, we believe we have not fully exploited the strong semantic reasoning capabilities of the VLMs, and there remains challenges in how to better incorporate, for example, multiple views of the scene for VLM reasoning and question answering in future work.

\smallskip \noindent \textbf{Miscalibrated question-image relevance score.} 
\cref{sec:stopping-criterion} introduces the question-image relevance score $\text{Rel}(x^t)$ \cref{eq:relevance} as a possible way of determining when the robot deems the current view sufficient for answering the question and stops exploration. However, from experiments we find this score very miscalibrated. \cref{fig:top-relevancy} shows the histogram of the highest relevance score within each scenarios (run with maximum time steps), where the answer prediction at the time step with the score is correct. Even though the question is answered correctly with the highest relevance score, the distribution of the scores is very wide --- majority of the successful scenarios has the top relevance  score lower than 0.5. Ideally the score should be high. This means the raw score alone cannot be used as a reliable indicator as the confidence for answering the question, thus motivating using multi-step conformal prediction (\cref{sec:multi-step CP}) to rigorously quantify the uncertainty instead.

\begin{figure}[h]
\centering\includegraphics[width=0.70\linewidth]{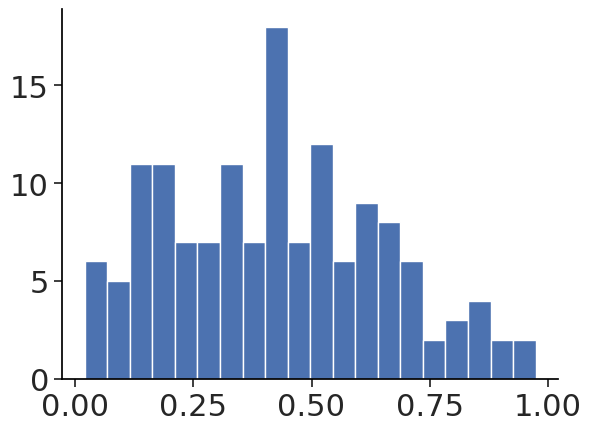}
    \vspace{-3pt}
    \caption{Histogram of the highest relevance score over steps from each successful scenario. Many scenarios have the top score lower than 0.5, and thus the raw score does not well reflect the confidence for answering the question.}
    \label{fig:top-relevancy}
    \vspace{-5pt}
\end{figure}

\begin{figure*}[t]
\begin{center}
\includegraphics[width=0.99\textwidth]{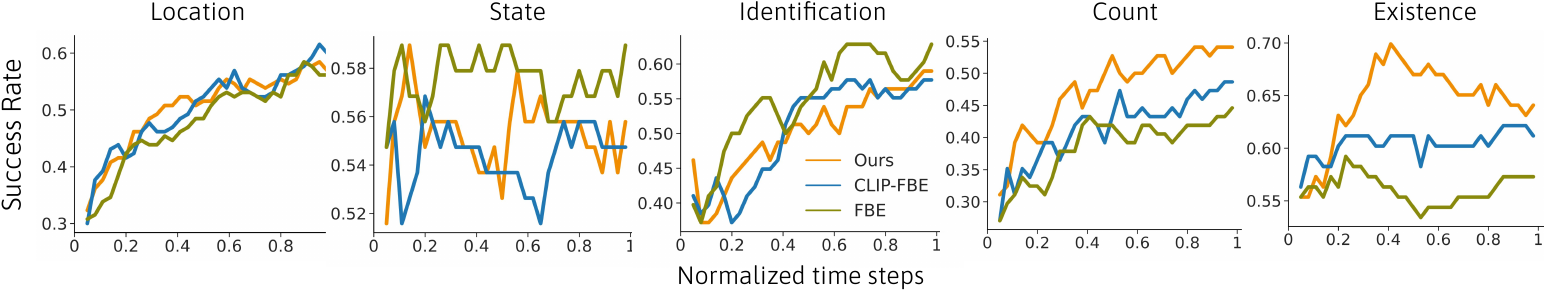}
\caption{Normalized time step taken vs. success rate in simulated experiments for comparing different exploration methods, in each of the five question categories: Location, State, Identification, Count, Existence. Our active exploration method shows the most improvement in Count and Existence, where the question provides a reasonable amount of information about the location of interest.}
\label{fig:category-comparison}
\vspace{-15pt}
\end{center}
\end{figure*}

\newpage
\smallskip \noindent \textbf{Semantic exploration results for each question category.} In \cref{fig:exploration-comparison}a we have shown the comparison between our method and two other exploration baselines (CLIP-FBE and FBE) in simulated experiments, using all scenarios from the HM-EQA dataset. Here in \cref{fig:category-comparison} we show the results for each of the five question categories (Location, State, Identification, Count, Existence) separately. We find our method shows the most improvement over baselines in Count and Existence. We believe the reason is that in these two type of questions, the question itself provides a reasonable amount of information about the location of interest (\eg `how many stools are there at the kitchen counter?' and `are there some towels in the bathroom'), and such information can be used by VLM to indicate possible exploration directions. In contrast, Location questions do not have such information (\eg `where is the piano?'), and the robot needs to explore most areas of the scene for answering them anyway. We also find State questions more difficult than the other categories, as the success rates do not improve with more time steps (\eg `Is the living room air conditioning turned on?'), since it tends to involve small objects that are difficult to see or require very close views. We also note our method does not show improvement for Identification questions --- the difference of these questions from Count and Existence is that, the questions only mention the \emph{object} of interest (\eg the piano) but not the \emph{location} of interest (\eg the bedroom), and this means identifying relevant exploration directions for them requires a higher degree of semantic reasoning of the VLM. We hope further improvement of the VLM can help better solve these scenarios.  

\smallskip \noindent \textbf{Semantic exploration results for varying scene sizes.} In \cref{fig:scene-size} we examine the success rate (using the maximum normalized time steps) achieved by different methods in scenes grouped by their sizes from \SI{150}{\metre\squared} to \SI{750}{\metre\squared}. The figure shows the success rate is roughly invariant to the scene size. This shows that our time step normalization is reasonable, and frontier-based exploration is essential to the different methods.

\begin{figure}[h]
\centering\includegraphics[width=0.70\linewidth]{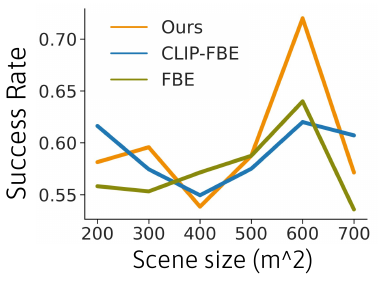}
    \vspace{-3pt}
    \caption{Scene size (in square meter) vs. success rate in simulated experiments for comparing different exploration methods. The scene size is binned with intervals of \SI{100}{\metre\squared}. The large spike at \SI{600}{\metre\squared} is caused by the relatively small number of scenarios in scenes around \SI{600}{\metre\squared}.}
    \label{fig:scene-size}
    \vspace{-5pt}
\end{figure}

\smallskip \noindent \textbf{Improving semantic exploration by considering view orientation.} View orientation can play an important role in determining the semantic value of a location in the scene, \eg the object on the sofa can be only seen from the front but not from the rear due to occlusion. Our proposed setup in \cref{subsec:semantic-value} does not consider this aspect; originally we alleviate the issue by restricting the robot to move in smaller steps such that one semantically relevant location might be seen by the robot from multiple view orientations (\cref{subsec:map}). In an additional experiment, we save the view orientation from individual observations in the semantic map, and apply bias to such orientation when sampling the next frontier to navigate to. We find this heuristic slightly improves the performance from 58.4\% success rate to 59.3\%. Future work could add additional dimensions to the semantic map corresponding to discretized view angles similar to affordance maps used in table-top manipulation tasks \cite{zeng2020tossingbot}.